\documentclass{article}

\usepackage{microtype}
\usepackage{graphicx}
\usepackage{booktabs} 

\usepackage{hyperref}

\usepackage[accepted]{icml2024}

\usepackage{amsmath}
\usepackage{amssymb}
\usepackage{mathtools}
\usepackage{amsthm}
\usepackage[capitalize,noabbrev]{cleveref}
\usepackage{subcaption}
\usepackage{multirow}
\usepackage{diffcoeff}
\usepackage{bbm}
\usepackage{wrapfig}
\usepackage{float}
\usepackage{array}
\theoremstyle{plain}
\newtheorem{theorem}{Theorem}[section]
\newtheorem{proposition}[theorem]{Proposition}
\newtheorem{lemma}[theorem]{Lemma}

\theoremstyle{definition}
\newtheorem{definition}[theorem]{Definition}

\theoremstyle{remark}

\usepackage[textsize=tiny]{todonotes}

\DeclareMathOperator{\indic}{\mathbbm{1}}
\DeclareMathOperator{\NN}{\mathbb{N}}
\DeclareMathOperator{\RR}{\mathbb{R}}
\let\P\relax
\DeclareMathOperator{\P}{\mathbb{P}}
\DeclareMathOperator{\Q}{\mathbb{Q}}
\DeclareMathOperator{\E}{\mathbb{E}}
\DeclareMathOperator{\B}{\mathcal{B}}
\DeclareMathOperator{\F}{\mathcal{F}}
\DeclareMathOperator{\G}{\mathcal{G}}

\DeclareMathOperator{\Prob}{Prob}
\DeclareMathOperator{\subg}{subG}
\newcommand{\DV}{\tilde{R}}
\newcommand{\DVnm}{\tilde{R}_{n,m}}
\newcolumntype{M}[1]{>{\centering\arraybackslash}m{#1}}
\newcommand{\defaultfontsize}{\f@size pt}
\newcommand{\semitransp}[2][35]{\textcolor{black!#1}{\footnotesize\selectfont #2}}
\definecolor{darkpastelgreen}{rgb}{0.01, 0.75, 0.24}

\icmltitlerunning{REMEDI: Corrective Transformations for Improved Neural Entropy Estimation}

\begin{document}

\twocolumn[
\icmltitle{REMEDI: Corrective Transformations for Improved Neural Entropy Estimation}



\icmlsetsymbol{equal}{*}

\begin{icmlauthorlist}
\icmlauthor{Viktor Nilsson}{equal,kth}
\icmlauthor{Anirban Samaddar}{equal,argonne}
\icmlauthor{Sandeep Madireddy}{argonne}
\icmlauthor{Pierre Nyquist}{kth,chalmers}
\end{icmlauthorlist}

\icmlaffiliation{kth}{Department of Mathematics, KTH Royal Institute of Technology, Stockholm, Sweden}
\icmlaffiliation{argonne}{Mathematics and Computer Science Division, Argonne National Laboratory, Chicago IL, USA}
\icmlaffiliation{chalmers}{Department of Mathematical Sciences, Chalmers University of Technology and University of Gothenburg, Gothenburg, Sweden}

\icmlcorrespondingauthor{Viktor Nilsson}{vikn@kth.se}

\icmlkeywords{Machine Learning, ICML}

\vskip 0.3in
]



\printAffiliationsAndNotice{\icmlEqualContribution} 

\begin{abstract}
    Information theoretic quantities play a central role in machine learning. 
    The recent surge in the complexity of data and models has increased the demand for accurate estimation of these quantities. 
    However, as the dimension grows the estimation presents significant challenges, with existing methods struggling already in relatively low dimensions. 
    To address this issue, in this work, we introduce \texttt{REMEDI} for efficient and accurate estimation of differential entropy, a fundamental information theoretic quantity.
    The approach combines the minimization of the cross-entropy for simple, adaptive base models and the estimation of their deviation, in terms of the relative entropy, from the data density.
    Our approach demonstrates improvement across a broad spectrum of estimation tasks, encompassing entropy estimation on both synthetic and natural data. 
    Further, we extend important theoretical consistency results to a more generalized setting required by our approach.
    We illustrate how the framework can be naturally extended to information theoretic supervised learning models, with a specific focus on the Information Bottleneck approach.
    It is demonstrated that the method delivers better accuracy compared to the existing methods in Information Bottleneck. In addition, we explore a natural connection between \texttt{REMEDI} and generative modeling using rejection sampling and Langevin dynamics.
\end{abstract}

\section{Introduction}

Information theoretic quantities such as entropy, cross-entropy, mutual information, relative entropy (Kullback-Leibler divergence), and conditional entropy are abundant in machine learning.
Many learning algorithms are derived from such quantities, and recent advances have revealed that they can provide learning objectives on their own \cite{tishby2015deep, alemi2016deep}, or in combination with other terms \cite{sarra2021renormalized, kingma2013auto}.
In some settings, an information theoretic objective may reduce to a simple expression in practical machine learning algorithms; for example: minimizing the forward relative entropy $R(\P || \Q)$ with respect to $\Q$, having samples $\{x_i\} \sim \P$, can be achieved by minimizing the negative log-likelihood. 
However, these quantities are usually difficult to estimate even in moderately high dimensions \cite{gao2018demystifying}.

In this work, we turn our attention toward the estimation of the differential entropy (DE).
This quantity appears in many places throughout machine learning, such as reinforcement learning \cite{haarnoja2017reinforcement}, unsupervised learning \cite{sarra2021renormalized}, the Information Bottleneck method \cite{alemi2016deep, kolchinsky2019nonlinear}, and dimensionality reduction \cite{faivishevsky2008ica}.
Often, differential entropy serves the role of something to be maximized, perhaps under some constraints, such as in the maximum entropy approaches \cite{jaynes1957information}, or the Information Bottleneck method \cite{tishby2000information, tishby2015deep}.
Therefore, it is advantageous that an estimator is differentiable with respect to the data, something not true for many estimators, e.g., $k$-nearest neighbor-based estimators.
There exist kernel-based plug-in estimators that are differentiable, however, they are prohibitively data inefficient in dimensions as low as 10 (see Chapter~20.3 in \cite{wasserman2004all}).

Recent works \cite{schraudolph2004gradient, pichler2022differential}, have introduced gradient-based learning objectives as upper bounds to the differential entropy, in an extension of classical kernel density estimation techniques \cite{rosenblatt1956remarks, parzen1962estimation, ahmad1976nonparametric}.
However, these estimators still lie in the class of plug-in 
estimators and are affected by the data inefficiency in large dimensions, see  Sec.~\ref{sec:limitofentropyestimators} and Appendix~\ref{sec:inefficiency-of-knife}.

To this end, our contributions are as follows ---
\vspace{-\topsep} 
\begin{itemize}
    \item  We introduce the \textit{Relative Entropy MixturE moDel corrective transformatIon} (\texttt{REMEDI}) approach that combines the strengths of recent advances in the estimation of information theoretic quantities
    \cite{pichler2022differential, belghazi2018mutual} to improve DE estimation.
    The approach takes modern plug-in entropy estimators and refines their estimates with corrections obtained via the Donsker-Varadhan formula.
    \item We present theoretical results proving the consistency of the proposed estimator, under the assumption that the data is sub-Gaussian, whereas existing related results put compactness assumptions on the support of the data.
    \item  We demonstrate the limitations of the current state-of-the-art plug-in entropy estimators in moderately high dimensions. We show that applying \texttt{REMEDI} corrections to the existing differential entropy estimators significantly improves entropy estimation in  
    benchmark datasets.
    \item  We discuss the application of our approach in supervised learning with the Information Bottleneck framework, to
    better estimate the mutual information between inputs and the latent space. On classification tasks with the MNIST, CIFAR-10, and ImageNet datasets, we show that \texttt{REMEDI} achieves better classification accuracy than state-of-the-art approaches.
    \item  We explore the generative modeling aspect of our approach by highlighting the connections of \texttt{REMEDI} with (i) rejection sampling, and (ii) stochastic differential equations. 
\end{itemize}

\section{Related works}

Estimators of differential entropy are usually classified as \textit{plug-in} estimates, \textit{sample-spacings} based estimates, and \textit{nearest neighbor} based estimates \cite{beirlant1997nonparametric}.
A classical estimator for the differential entropy is Parzen-window estimation, which is provably consistent under weak assumptions \cite{ahmad1976nonparametric}.
It is known that Parzen-window approximation is data-inefficient, see \cite{wasserman2004all}. 
Modern extensions of \cite{ahmad1976nonparametric} come from \cite{schraudolph2004gradient}, where each kernel has its precision matrix parametrized by its lower Cholesky factors, with respect to which it is also differentiable, which enables gradient-based learning.
In \cite{pichler2022differential}, this is further extended to involve parametrized means and weights for each kernel.
The three above-mentioned approaches are all plug-in estimators \cite{beirlant1997nonparametric}, where the latter two involve a training step, whereby a mixture of multivariate normal distributions are fitted to the data using gradient-based training.

In recent studies \cite{belghazi2018mutual}, a new method has been introduced for computing the mutual information $\operatorname{MI}(X; Y)$ between two random variables $X, Y$, based on neural networks and the Donsker-Varadhan representation formula.
The method exploits that the mutual information is the same as the relative entropy between $\P_{XY}$ and $\P_X \otimes \P_Y$, which makes the Donsker-Varadhan formula a lower bound for the mutual information.
However, recently \cite{mcallester2020formal} it has been shown that such bounds are statistically difficult to approximate when $\P_{XY}$ and $\P_X \otimes \P_Y$ are too different.
Using a base distribution, the same type of bounds on the relative entropy can be exploited for computing the differential entropy, then appearing as upper bounds, see Sec.~\ref{sec:method}. Recent studies
\cite{park2021deep, aharoni2022density} have taken this direction to neural differential entropy estimation.
However, the authors only consider very simple base distributions, which implies poor sample efficiency, likely introducing the problems presented in \cite{mcallester2020formal}.

In this work, we explore the combination of a more flexible mixture of Gaussian base distribution with the Donsker-Varadhan type bound to propose an improved neural density estimator. 
In addition, we theoretically prove that our estimator is consistent, under weaker assumptions on the data distribution compared to previous works \cite{belghazi2018mutual}.

\section{Method}\label{sec:method}
In this section, we present the \texttt{REMEDI} approach for entropy estimation. We state the limitations of the existing entropy estimation approaches and discuss the ways \texttt{REMEDI} addresses these limitations. Furthermore, we present the optimization algorithm for solving the loss function using finite data samples.

\subsection{Preliminary concepts}
In this section, we define the necessary concepts that will be used throughout the paper.

Let $\P$ and $\Q$ be two probability measures, such that $\P \ll \Q$, on a Polish space $\mathcal{X}$ and its Borel $\sigma$-algebra $\Sigma$.
$\mathcal{X}$ may be taken to be $\RR^d$ or a (closed) subset thereof.

The relative entropy, or Kullback-Leibler divergence, from $\P$ to $\Q$ is defined as
\begin{equation}
    R(\P \parallel \Q) = \E^{\P}\left[\log\left(\frac{d\P}{d\Q}\right)\right] = \int \log\left(\frac{d\P}{d\Q}\right) d\P
\end{equation}

\noindent The Donsker-Varadhan representation for the relative entropy states that
\begin{equation}\label{eq:DV-repr}
    R(\P \parallel \Q) = \sup_{T: \mathcal{X} \to \RR} \E^{\P}[T] - \log\E^{\Q}[e^T],
\end{equation}
where the supremum is taken over the class of continuous bounded functions, or the class of Borel-measurable functions from $\mathcal{X}$ to $\RR$ \cite{donsker1983asymptotic, budhiraja2019analysis}.

Assume that $\P, \Q \ll \lambda$
, where $\lambda$ is the Lebesgue measure.
Then they have densities $p, q$.
The definition of the differential entropy of $\P$ is
\begin{equation}
    H(\P) = - \E^{\P}\left[\log \frac{d\P}{d\lambda}\right] = - \E^{\P}\left[\log p \right].
\end{equation}

\subsection{Limitations of existing entropy estimation approaches}
\label{sec:limitofentropyestimators}

Given $n$ samples $x_1, x_2, ..., x_n$ from the data distribution $p(x)$ an oracle estimator of entropy is the Monte Carlo estimator: $\hat{H}_{Oracle}(x) = n^{-1} \sum_{i=1}^{n} - \log p(x_i)$. However, in practice, we rarely know the true data density $p(x)$. Therefore, a large body of literature focuses on finding an accurate plug-in estimator of the density $\hat{p}(x)$ which can replace $p(x)$ in the oracle estimator.

The Parzen window density estimation for $p(x)$ amounts to specifying a bandwidth $h$ and then letting $\hat{p}(x) = \frac{1}{nh}\sum_{i=1}^n \phi(\frac{x - x_i}{h})$, where $\phi$ is the standard isotropic normal density in $\RR^d$.
In \cite{schraudolph2004gradient} this is improved by parametrizing the positive definite covariance matrices of the individual kernels by their lower Cholesky factors, allowing for gradient-based learning of them.

\begin{figure}
  \begin{subfigure}{0.2375\textwidth}
    \centering
    \includegraphics[width=\linewidth]{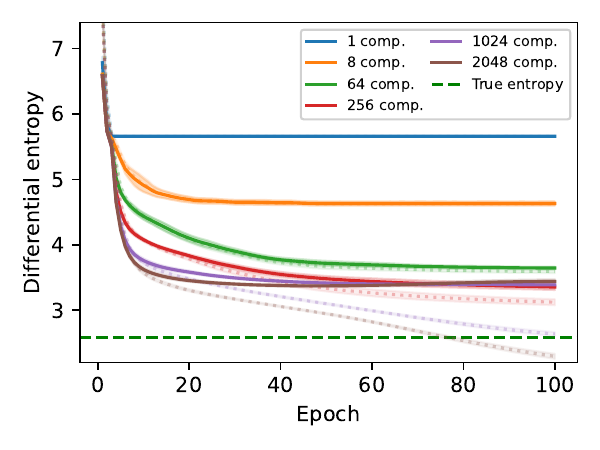}
    \caption{Triangle}
    \label{fig:triangle_8d_varcomps}
  \end{subfigure}%
  \begin{subfigure}{0.2375\textwidth}
    \centering
    \includegraphics[width=\linewidth]{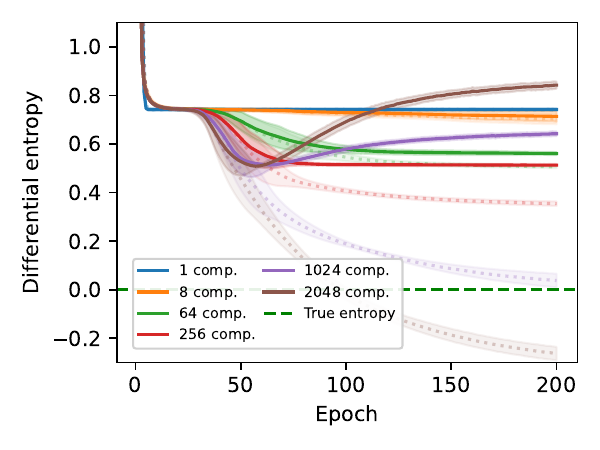}
    \caption{Ball}
    \label{fig:ball_8d_varcomps_losses}
  \end{subfigure}
  \caption{KNIFE training curves with error bars on $8$-dimensional triangle and uniform ball datasets. It is observed that increasing the number of components $M$ for KNIFE leads to overfitting in both datasets.}
  \label{fig:8d_varcomps_losses}
\end{figure}

In \cite{pichler2022differential}, the authors propose a plug-in density estimator, KNIFE, that is a learnable mixture of $M$ multivariate Gaussian kernels. 
However, the number of mixture components, $M$ used in KNIFE, is treated as a hyperparameter to be optimized using discrete grid search. 
In our experiments, we found that training KNIFE requires delicate tuning of the number of components $M$. 
In the training, KNIFE minimizes the cross-entropy loss function which is an upper bound to the true entropy i.e. $H(\mathbb{P}) \leq \mathcal{L}_{\mathrm{KNIFE}}$ (Eq.~6 in \cite{pichler2022differential}). 
In Fig.~\ref{fig:8d_varcomps_losses}, we present the validation (solid) and training (dotted) loss curves by fitting KNIFE on $8$-dimensional triangle and uniform ball (see Sec.~\ref{sec:datasets}) datasets. 
We observe that increasing $M$ leads to overfitting on both datasets. 
In addition, the best estimate of the entropy on the validation set lies significantly away from the true entropy (green dotted line). The estimation errors are even worse when the data dimension $20$ (see Fig.~\ref{fig:20d_varcomps_losses} in Appendix~\ref{sec:additional-synthetic-benchmarks}).
This behavior is in line with the data inefficiency of the simpler kernel density estimators in moderately large dimensions \cite{wasserman2004all}.
To this end, we propose \texttt{REMEDI} that applies a correction to any simple learnable base density. 
We show that the \texttt{REMEDI} estimator is theoretically consistent and is tractable using existing optimization tools.

\subsection{The \texttt{REMEDI} approach}

In this section, we derive a bound that is tight to the true entropy $H(\mathbb{P})$ and can be parametrized for efficient optimization. Using Eq.~\eqref{eq:DV-repr}, we have,
\vspace{-\topsep}
\begin{equation}\label{eq:entropy-est-DV}
\begin{split}
    H(\P) &= -\E^{\P}\left[\log p \right] \\
    &= -\E^{\P}\left[\log q \right] - \E^{\P}\left[\log \left(p/q\right) \right] \\
    &= -\E^{\P}\left[\log q \right] - R(\P \parallel \Q) \\
    &= -\E^{\P}\left[\log q \right] - \sup_{T: \mathcal{X} \to \RR} (\E^{\P}[T]  - \log\E^{\Q}[e^T]) \\
\end{split}
\end{equation}
\noindent The first term in Eq.~\eqref{eq:entropy-est-DV} is known as the cross-entropy from $\P$ to $\Q$.
Working with Eq.~\eqref{eq:entropy-est-DV}, we have,
\vspace{-\topsep}
\begin{equation}\label{eq:entropy-est-DV-2}
    H(\P) \leq -\E^{\P}\left[\log q \right] - \left(\E^{\P}[T] - \log\E^{\Q}[e^T]\right) 
   \eqcolon \mathcal{L}_{\texttt{REMEDI}}
\end{equation}
where equality in Eq.~\eqref{eq:entropy-est-DV-2} holds when taking the infimum on the right-hand side with respect to $T$. 
It can be shown that minimizing Eq.~\eqref{eq:entropy-est-DV-2} is equivalent to minimizing a cross-entropy loss between the Gibbs density $\tilde{p}$ induced by $T$ and true density $p$. 
Proposition~\ref{Thm:DV_loss} in the appendix provides more insight into this bound in Eq.~\eqref{eq:entropy-est-DV-2}.

Proposition~\ref{Thm:DV_loss} implies that, (i) for any choice of $\Q$ optimizing RHS of Eq.~\ref{eq:entropy-est-DV-2} is equivalent to optimizing a cross-entropy loss between $\tilde{p}$ and true density $p$, and (ii) the optimal solution is reached when the associated density is equal to the true density. Therefore, we use Eq.~\eqref{eq:entropy-est-DV-2} as a loss function, denoted by $\mathcal{L}_{\texttt{REMEDI}}$, which we minimize to estimate $H(\mathbb{P})$.

\subsection{Algorithm}

Given a set of $n$ samples $\{x_i\}_{i=1}^n$ from the data distribution $\P$ and $m$ independent samples $\{\tilde{x}_j\}_{j=1}^m$ from $\Q$, we can minimize $\mathcal{L}_{\texttt{REMEDI}}$ with standard gradient-based optimization tools, by considering its empirical counterpart: 
\begin{equation}\label{eq:emploss}
\begin{split}
    \hat{\mathcal{L}}_\texttt{REMEDI} &= \underbrace{\frac{1}{n} \sum_{i=1}^{n} - \log q(x_i)}_{\hat{\mathcal{L}}_{\mathrm{KNIFE}}}  \\
    &- \underbrace{\left(\frac{1}{n} \sum_{i=1}^{n} T(x_i) - \log\left(\frac{1}{m} \sum_{i=1}^{m} e^{T(\tilde{x}_i)}\right)\right)}_{\hat{\mathcal{L}}_{\mathrm{DV}}}. \\
\end{split}
\end{equation}
The loss function $\hat{\mathcal{L}}_{\texttt{REMEDI}}$ has two components, (i) the cross-entropy loss $\hat{\mathcal{L}}_{\mathrm{KNIFE}}$ for training a KNIFE base distribution, and (ii) the Donsker-Varadhan loss $\hat{\mathcal{L}}_{\mathrm{DV}}$. 
Here, $\tilde{x}_1,..., \tilde{x}_n$ represents $n$ samples from the base distribution. 
Although we chose KNIFE as the base distribution $\Q$ it can be replaced by any distribution with a tractable likelihood and efficient sampling scheme. 
In our experiments, we found that KNIFE with a few components is a good candidate for the base distribution. 
Following \cite{belghazi2018mutual} we parametrize $T$ as a neural network, which provides a flexible class for function approximation \cite{hornik1989multilayer}. 

We present the implementation details of \texttt{REMEDI} in Algorithm~\ref{REMEDI}. The parameters $\theta_{\text{KNIFE}},\phi_{T}$ denote the parameters of KNIFE and the neural network $T$ respectively. Note that, optimizing $\hat{\mathcal{L}}_{\mathrm{DV}}$ using naive gradient routines in, e.g., Pytorch \cite{paszke2019pytorch} introduces bias in the stochastic gradient estimation of the log of expectation term in the DV loss (see Eq.12 in \cite{belghazi2018mutual}).
To alleviate this we use large batch-sizes and exponential moving average \cite{belghazi2018mutual} in our experiments.

\begin{algorithm}
\caption{\texttt{REMEDI} algorithm}\label{REMEDI}
\begin{algorithmic}[1]
\STATE Draw $n$ minibatch sample $(x_1, ..., x_n)$ from $p(x)$  
\STATE Initialize $(\theta_{\text{KNIFE}},\phi_{T})$
\FOR{$k_1$ epochs}
\STATE Update $\theta_{\text{KNIFE}}$ optimizing cross-entropy loss using the samples 
\ENDFOR
\STATE Draw $m$ samples $(\Tilde{x}_1, ..., \Tilde{x}_m)$ from KNIFE
\FOR{$k_2$ epochs}
\STATE Use $(x_1, ..., x_n)$ and $(\Tilde{x}_1, ..., \Tilde{x}_m)$
to update $\phi_{T}$ by optimizing the loss $\hat{\mathcal{L}}_{\mathrm{DV}}$ 
\ENDFOR
\STATE \textbf{return} $\hat{\theta}_{\text{KNIFE}},\hat{\phi}_{T}$
\end{algorithmic}
\end{algorithm}

\subsection{Theoretical results}

To assess the validity of the \texttt{REMEDI} approach to entropy estimation, we show in Appendix~\ref{sec:consistency} that, under weak conditions on $\P$ and $\Q$, the estimator satisfies an appropriate form of consistency.
This requires considering the limit of increasing both the amount of samples $n$ taken from $\P$ and $m$, taken from $\Q$.
By the law of large numbers, $\hat{\mathcal{L}}_{\mathrm{KNIFE}} = n^{-1} \sum_{i=1}^n -\log q(x_i) \to C(\P || \Q)$, and what needs to be shown for consistency is that, under appropriate conditions, $\hat{\mathcal{L}}_{DV} \to R(\P || \Q)$.
This is the done in Theorem~\ref{thm:remedi-consistency-non-compact}.
Then $\hat{\mathcal{L}}_{\mathrm{REMEDI}}$, by Eq.~\eqref{eq:emploss}, converges to $H(\P)$.

An analogous result for mutual information estimation is stated and proved in \cite{belghazi2018mutual}.
However, they require that $\mathcal{X}$ be compact.
This is much too strong for our purposes, since any linear combination or convolution of Gaussians has infinite tails.
For example (see Sec.~\ref{sec:experiments}), the two moons dataset has this property and, even more importantly, a common modeling choice for the latent space conditional distribution in the Information Bottleneck framework is to use a multivariate normal.

\section{Experiments}\label{sec:experiments}

In this section, we perform experiments to evaluate \texttt{REMEDI} in various unsupervised and supervised tasks ranging from differential entropy estimation, deep latent variable models, and generative models. 
We demonstrate using synthetic and natural datasets that \texttt{REMEDI} can easily be applicable in these tasks.

\subsection{Entropy estimation}

Entropy estimation is an important step in many real-world datasets. However, complex intrinsic features of the data distribution, e.g. high-dimensionality, and multi-modality can make the task challenging. 
In this section, we apply \texttt{REMEDI} for the entropy estimation of two such complex datasets: (1) Two moons, and (2) Triangle dataset \cite{pichler2022differential} (see Appendix~\ref{sec:additional-synthetic-benchmarks} for additional benchmarks). 
On these datasets, we compare \texttt{REMEDI} with KNIFE \cite{pichler2022differential} the current state-of-the-art approach for entropy estimation. 
Furthermore, we analyze the inner workings of how \texttt{REMEDI} corrects the base distribution. 

\subsubsection{Two moons}

We apply our method to the popular two moons dataset from Scikit-learn \cite{scikit-learn}.
Here for illustration, we use 8 components for the KNIFE base-distribution.
In Fig.~\ref{fig:two-moons-loss} we see that the training process for KNIFE planes out around its final estimate at around $0.45$.
Here, the \texttt{REMEDI} takes over and manages to push the estimate down below $0.30$. 
This is close to the true entropy of the dataset; here, the entropy offers no closed-form expression but an oracle computation, using kernel density estimation on one million samples, yields a value close to $0.29$, see Appendix~\ref{sec:two-moons-dataset}.
\begin{figure*}[!ht]
	\centering
	\begin{subfigure}{0.30\textwidth}
    	\centering
        \includegraphics[width=\linewidth]{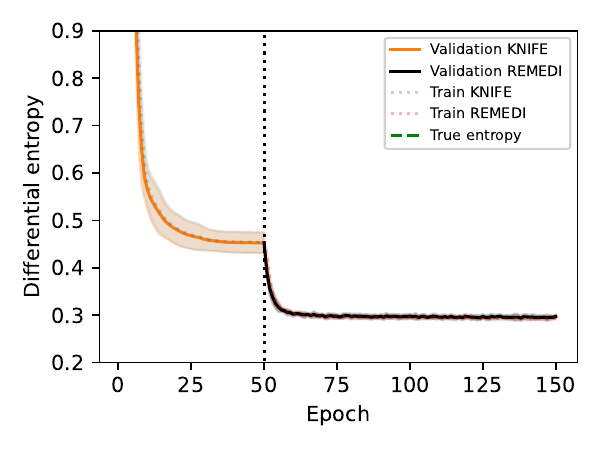}
    	\caption{Training curve}
        \label{fig:two-moons-loss}
	\end{subfigure}
	\begin{subfigure}{0.33\textwidth}
    	\centering
        \includegraphics[width=\linewidth]{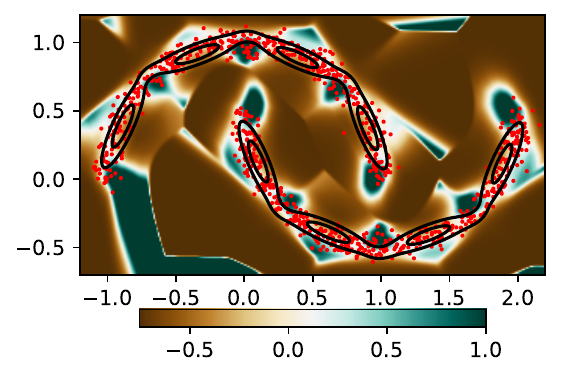}
    	\caption{$T$ (image) vs. $q$ (contours).}
        \label{fig:two-moons-correction}
    \end{subfigure}
    \begin{subfigure}{0.33\textwidth}
    	\centering
        \includegraphics[width=\linewidth]{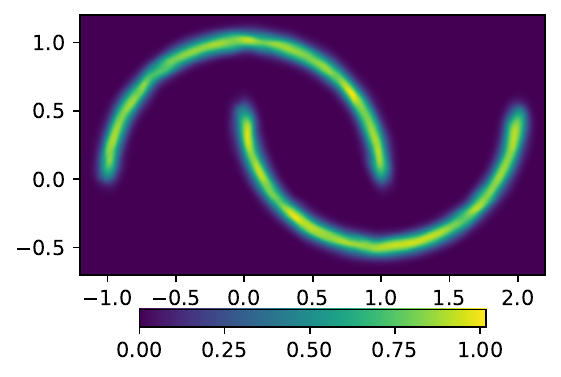}
    	\caption{Unnormalized density $q e^{T}(x)$}
        \label{fig:two-moons-qT}
    \end{subfigure}
    \caption{Results on two moons dataset. In the middle we see what direction (positive or negative) \texttt{REMEDI} affects the base distribution. To the right is the unnormalized distribution implied by $q(x) e^{T(x)}$.}
    \label{fig:two-moons}
\end{figure*}

To facilitate a better understanding of the learning mechanism of \texttt{REMEDI}, Fig.~\ref{fig:two-moons-correction} shows a visualization of the correction given by the method, compared to the contours of KNIFE, i.e. $\Q$, and samples from $\P$.
We can see that \texttt{REMEDI} reinforces low-probability regions of the data, with respect to $\Q$, by putting higher relative corrections there.
This is in line with the interpretation of $T$ as learning the (log) unnormalized Radon-Nikodym derivative $\frac{d \P}{d \Q}$.
Exploiting this notion, we provide the corresponding learned density plot in Fig.~\ref{fig:two-moons-qT}.

\subsubsection{Triangle mixture}
In this section, we compare \texttt{REMEDI} to KNIFE on the triangle mixture dataset \cite{pichler2022differential}. 
This dataset consists of samples from a mixture of triangle distributions resulting in many modes. 
Applying \texttt{REMEDI} to the one-dimensional version of the dataset yields Fig.~\ref{fig:loss-triangle-1d}.
Since the normalizing constant is easily computed in one dimension, the normalized learned distribution along with the true density, is shown in Fig.~\ref{fig:pdf-triangle-1d}.
\begin{figure}[]
	\centering
	\begin{subfigure}{0.4\textwidth}
    	\centering
        \hspace{-1.5em}
        \includegraphics[width=\linewidth]{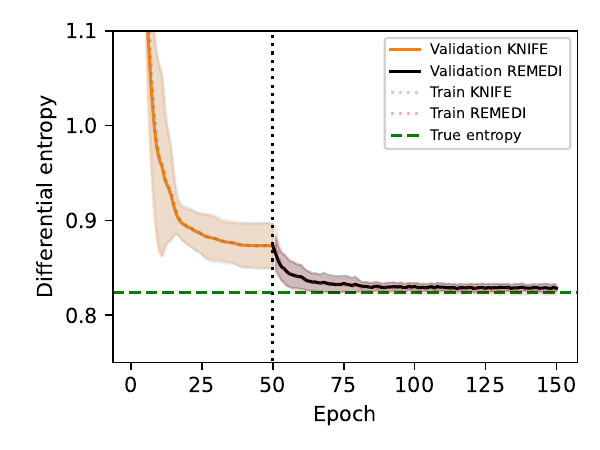}
        \caption{Training curve with standard deviations from $10$ repetitions.}
        \label{fig:loss-triangle-1d}
    \end{subfigure}
	\begin{subfigure}{0.4\textwidth}
    	\centering
        \includegraphics[width=\linewidth]{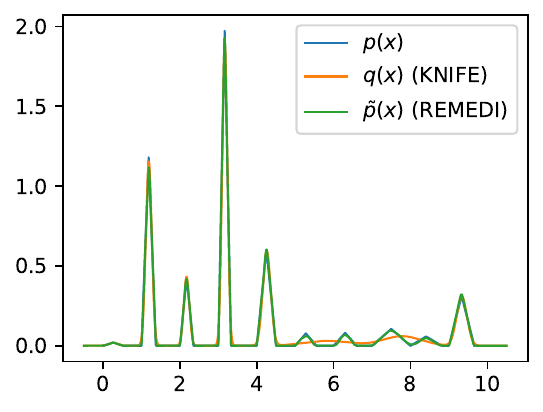}
        \caption{Comparison of \texttt{REMEDI} estimated pdf with KNIFE.}
        \label{fig:pdf-triangle-1d}
	\end{subfigure}
    \caption{Results on one-dimensional triangle dataset. On the bottom is the data distribution, compared to the density that KNIFE and \texttt{REMEDI} (up to a constant) has learned.}
    \label{fig:triangle-1d}
\end{figure}
The high-dimensional versions of this dataset constitute challenging targets for our model.
In \cite{pichler2022differential}, the authors showcase decent results in 8 dimensions, although still missing the target by several integer points.
This is the eight-fold product distribution of one-dimensional two-component distributions, resulting in a $2^8$-modal distribution, see Appendix~\ref{sec:datasets}.
The true entropy if this distribution is $2.585$.
Applying a 16-component KNIFE estimator to this task results in an estimates of around $4.36$, where the improvement stops.
Adding \texttt{REMEDI} on top of this model improves this estimate considerably to $3.08$, 
see Fig.~\ref{fig:loss-triangle-8d} for the full behavior.
Both estimates can be further improved by adding more components to the KNIFE-based $\Q$.
In Appendix~\ref{sec:additional-synthetic-benchmarks}, 
we show that KNIFE with an increasing number of components easily fails before reaching a good estimate, due to overfitting issues, while using fewer with a \texttt{REMEDI} correction is much more efficient.

\begin{figure}[]
    \centering
    \includegraphics[width=0.64\linewidth]{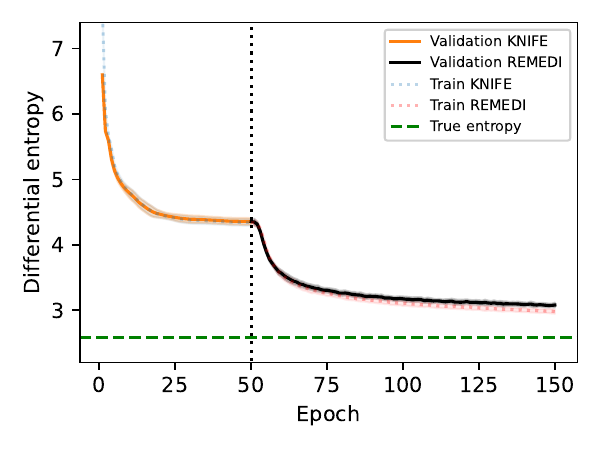}
    \caption{Training curve from the 8-dimensional triangle dataset. The horizontal dashed line indicates where the KNIFE training phase ends and \texttt{REMEDI} takes over.}
    \label{fig:loss-triangle-8d}
\end{figure}
\begin{table}[]
    \centering
    \begin{tabular}{|l|l|l|}
    \hline
        Methods & Entropy estimate \\
        \hline
        KNIFE &  $4.3563 \pm 0.0528$ \\
        \hline
        \texttt{REMEDI} & $3.0798 \pm 0.0368$ \\
        \hline
        True & $2.5852$ \\ 
        \hline
    \end{tabular}
    \caption{Estimates on 8-dimensional triangle dataset.}
    \label{tab:estimates-triangle-8d}
\end{table}

\subsection{Application to Information Bottleneck}

Information Bottleneck (IB) \cite{tishby2000information} is a popular latent variable model that aims to learn a representation $Z$ from the input $X$ that is maximally compressive of $X$ and maximally predictive of the output $Y$. In practice, the learning problem is posed as a maximization problem \cite{tishby2015deep} of the IB loss function:
\begin{equation}
   \mathcal{L}_{IB}(Z) = \mathrm{MI}(Z; Y) - \beta \ \mathrm{MI}(X; Z)
  \label{eq:IB_lagrangian}
\end{equation}
The loss function in Eq.~\eqref{eq:IB_lagrangian} is in terms of the two mutual information quantities. $\operatorname{MI}(Z;Y)$ measures the predictive information contained in $Z$ and $\operatorname{MI}(X;Z)$ measures the information about $X$ contained in $Z$. Maximizing 
Eq.~\eqref{eq:IB_lagrangian} implies compressing the inputs while simultaneously maximizing the predictive information in the representation $Z$. The Lagrange multiplier $\beta$ serves to control the amount of compression.  

The Information Bottleneck has been applied in deep learning \cite{alemi2016deep, achille2018information} to learn compressed representation from high-dimensional inputs e.g. images and texts that are highly predictive of the low-dimensional targets e.g. labels. The usual practice is to parametrize the representation $Z$ by a stochastic encoder $p_\psi(z|x)$. Following \cite{alemi2016deep}, we assume $p_\psi(z|x)$ to be a multivariate Gaussian distribution with mean $\mu(X)$ and a diagonal covariance matrix $\Sigma(X)$. However, it is infeasible to calculate the mutual information $\operatorname{MI}(X; Z)$, especially for complex datasets.
Therefore, there are many methods proposed in the literature for accurate estimation of $\operatorname{MI}(X; Z)$ ranging from parametric \cite{alemi2016deep}, non-parametric (kernel density based) \cite{kolchinsky2019nonlinear,pichler2022differential}, and adversarial f-divergences \cite{belghazi2018mutual,zhai2022adversarial}.  

The application of \texttt{REMEDI} to estimate $\operatorname{MI}(X;Z)$ is straightforward. From the decomposition $\operatorname{MI}(X;Z) = H(Z) - H(Z|X)$ and the fact that we can analytically derive $H(Z|X)$, estimation of the mutual information boils down to accurately estimating the entropy $H(Z)$. We follow the Algorithm~\ref{REMEDI} to apply \texttt{REMEDI} by replacing $p(x)$ with the stochastic encoder $p_\psi(z|x)$. In applying \texttt{REMEDI}, we assume a coordinate-wise independent isotropic Gaussian and a 10-component KNIFE as the choices for the base distribution. Our choice for the number of components in KNIFE is based on the number of classes in the data (see Fig.2 in \cite{alemi2016deep}) and available GPU memory to fit the parameters. Additional details about the implementation are provided in the Appendix~\ref{sec:additionaldetailsexperiments}.

\begin{figure*}
    \centering
	\centering
	\begin{subfigure}{0.28\textwidth}
	\centering
	\includegraphics[width=\linewidth]{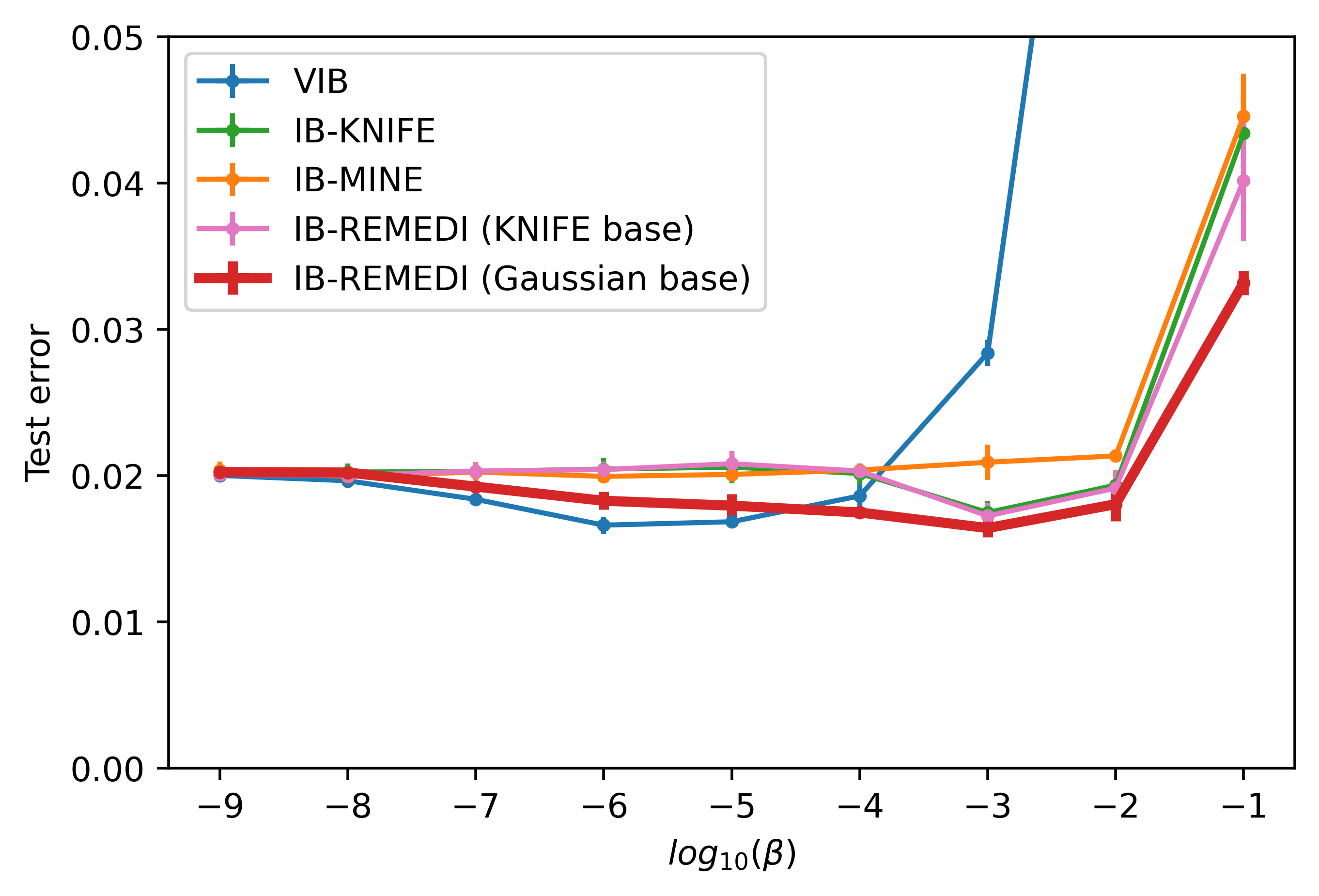}
	\caption{MNIST}
	\end{subfigure}
        \hspace{1cm}
	\begin{subfigure}{0.28\textwidth}
	\centering
	\includegraphics[width=\linewidth]{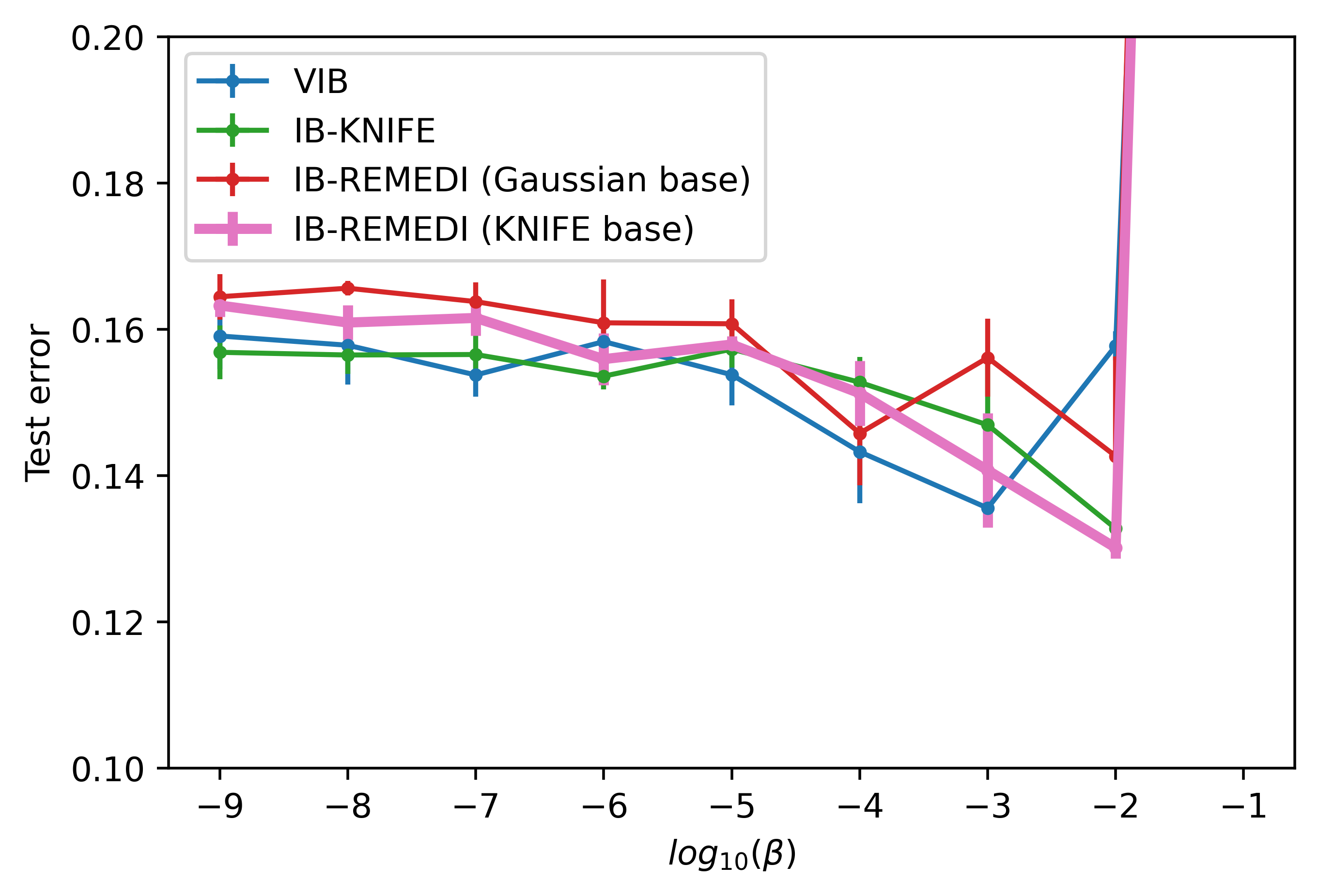}
	\caption{CIFAR10}
	\end{subfigure}
        \hspace{1cm}
	\begin{subfigure}{0.28\textwidth}
	\centering
	\includegraphics[width=\linewidth]{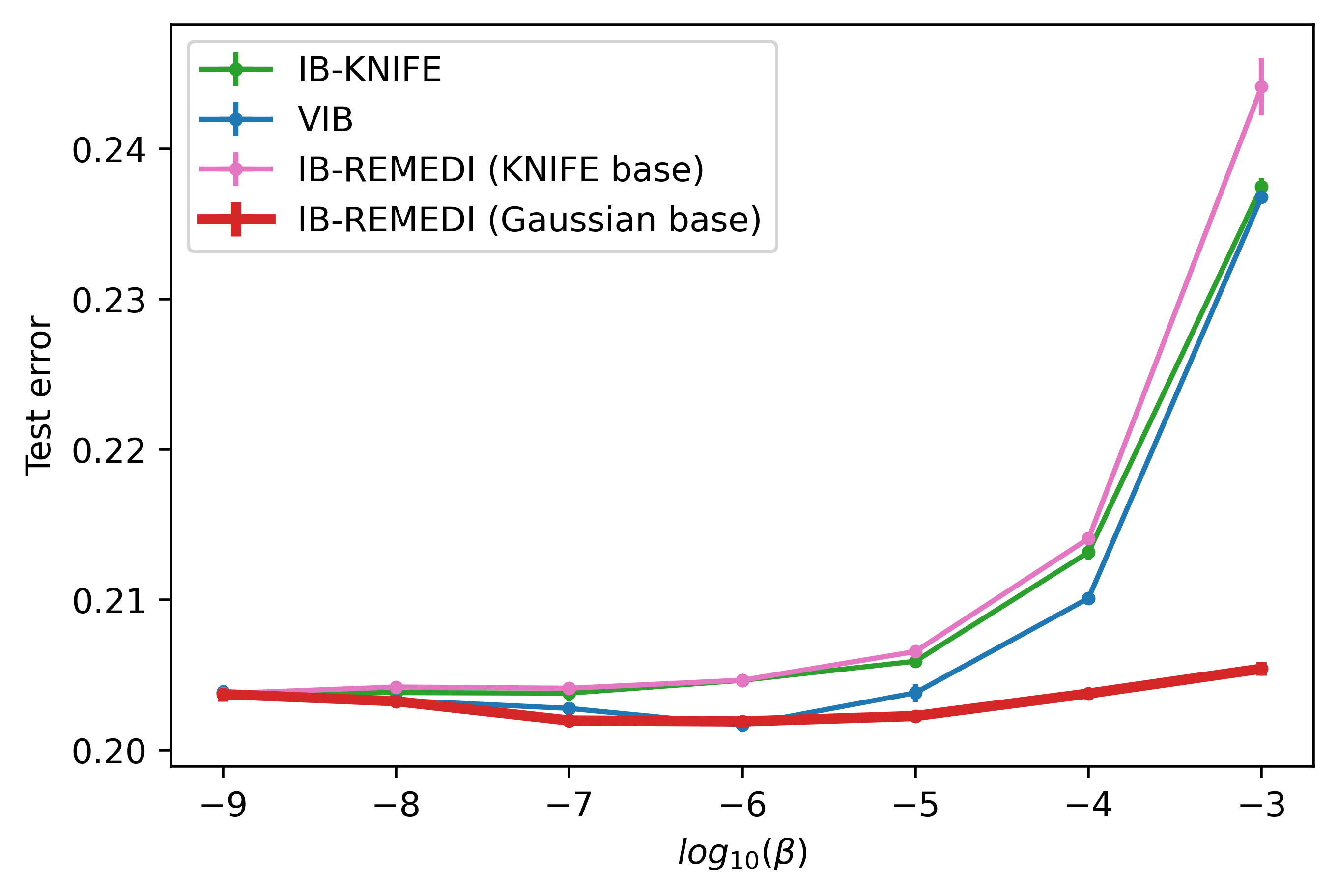}
	\caption{ImageNet}
	\end{subfigure}
    \caption{Plot showing test error of the Information Bottleneck methods vs $\beta$ on benchmark image classification datasets (error bars represent standard deviations). For most $\beta$ values, consistently \texttt{REMEDI} performs better than other methods on MNIST and ImageNet. On CIFAR10, the classification errors are similar for all the methods. However, \texttt{REMEDI} exhibits the lowest classification error across the $\beta$ values.}
    \label{fig:classification}
\end{figure*}

We perform experiments to compare \texttt{REMEDI} against state-of-the-art mutual information estimation approaches: VIB \cite{alemi2016deep}, KNIFE \cite{pichler2022differential}, and MINE \cite{belghazi2018mutual}. In our experiments, we followed the open-source implementation of these approaches (the code for our implementation is provided in the supplementary materials). We evaluate \texttt{REMEDI} and other approaches in image classification tasks on MNIST, CIFAR-10, and ImageNet datasets. In our experiments, we use the standard training and test splits for these datasets and closely follow the network architectures from \cite{pmlr-v206-samaddar23a}. We run all methods for three different seeds. See Appendix~\ref{sec:additionaldetailsexperiments} for additional details.

\paragraph{Classification accuracy:} In this section, we compare the classification accuracy of each method on the test splits of the three datasets. For evaluation metrics, we compute the test set classification error and the log-likelihood ($\propto \operatorname{MI}(Z;Y)$). We present the "1 shot eval" from \cite{alemi2016deep} where we take one sample from the encoder and pass it to the decoder for prediction.

In Fig.~\ref{fig:classification}, we plot the classification errors against the Lagrange multiplier $\beta$ for the three datasets. On MNIST and ImageNet, we observe that \texttt{REMEDI} consistently exhibits the lowest classification errors for most of the $\beta$ values. On CIFAR10, all methods are seen to perform similarly, however, \texttt{REMEDI} exhibits the lowest classification error across the $\beta$ values. Across the datasets, we observe improvements in accuracy by \texttt{REMEDI}, especially around the values where the classification errors start to increase. This region is interesting because it contains the \textit{minimum necessary information} (MNI) point \cite{fischer2020conditional} where the model retains the necessary information from the inputs to predict the target and minimizes the redundant information from the inputs. Additionally, we present similar plots based on log-likelihood in Appendix~\ref{sec:results_loglike} where the conclusions remain unchanged.  

In Table~\ref{tab:test_MNISTCIFAR}, we present the best test accuracy and corresponding log-likelihood across $\beta$ for all the methods on the three datasets. On MNIST and CIFAR10, \texttt{REMEDI} performs better than all methods in terms of the metrics. On ImageNet, \texttt{REMEDI} performs similarly to the VIB method. However, it learns significantly less information about the inputs measured (160.67 bits) by $\widehat{\operatorname{MI}}(X;Z)$ than the VIB (237.71 bits). We note that on CIFAR10 and ImageNet we found stability issues with the MINE implementation perhaps due to the exponential term in the Donsker-Varadhan lower bound \cite{mcallester2020formal}.

In these experiments, we present results of two different base distributions for \texttt{REMEDI}. We found that choosing an independent isotropic Gaussian base distribution is computationally more efficient than choosing a trainable base distribution such as KNIFE. Additionally, our results indicate that choosing a Gaussian base distribution increases the classification accuracy of \texttt{REMEDI} especially on MNIST and ImageNet.

\begin{table*}
\centering
\resizebox{\textwidth}{!}{   
\begin{tabular}{|c|c|c|c|c|c|c|}
   \hline
   \multirow{2}{1.5cm}{\textbf{Methods}} & \multicolumn{2}{c|}{\textbf{MNIST}} & \multicolumn{2}{c|}{\textbf{CIFAR-10}} &
   \multicolumn{2}{c|}{\textbf{ImageNet}}\\
   \cline{2-7}
   & \textbf{Acc} \% & \textbf{LL} & \textbf{Acc} \% & \textbf{LL} & \textbf{Acc} \% & \textbf{LL}\\
	\hline
       KNIFE \cite{pichler2022differential} & 98.25 \semitransp{(0.093)} & 3.22 \semitransp{(0.005)}  & 86.73 \semitransp{(0.394)}  & 2.57 \semitransp{(0.007)}   & 79.62 \semitransp{(0.053)} & 8.55 \semitransp{(0.006)} \\
       MINE \cite{belghazi2018mutual} & 98.01 \semitransp{(0.040)}  & 3.22 \semitransp{(0.001)}  & - & - & - & -\\
	VIB \cite{alemi2016deep} & 98.34 \semitransp{(0.070)} & \textbf{3.24} \semitransp{(0.0003)}  & 86.45 \semitransp{(0.135)}  & 2.60 \semitransp{(0.013)}  & \textbf{79.83} \semitransp{(0.061)} & \textbf{8.65} \semitransp{(0.003)} \\
        \hline
        \texttt{REMEDI} (Gaussian base) & \textbf{98.36} \semitransp{(0.076)}  &  \textbf{3.24} \semitransp{(0.003)}   & 85.74 \semitransp{(1.671)}  &  \textbf{2.61} \semitransp{(0.056)} & 79.81 \semitransp{(0.047)} & 8.63 \semitransp{(0.002)}  \\
	\texttt{REMEDI} (KNIFE base) & 98.28 \semitransp{(0.105)}  & 3.22 \semitransp{(0.004)}   & \textbf{86.99} \semitransp{(0.176)}  & 2.57 \semitransp{(0.005)}  & 79.61 \semitransp{(0.015)} & 8.55 \semitransp{(0.006)} \\
   \hline
   \end{tabular}
   }
   \caption{Comparison of the best test accuracy and corresponding log-likelihood across the $\beta$ values for MNIST, CIFAR-10, and ImageNet (standard deviations in the parenthesis). In terms of the classification accuracy, on ImageNet \texttt{REMEDI} performs close to the best performing method. However, on MNIST and CIFAR10 \texttt{REMEDI} shows improvement over the state-of-the-art IB method.}
    \label{tab:test_MNISTCIFAR}
\end{table*}

\paragraph{Analysis of the latent space}
To understand the learning mechanism of \texttt{REMEDI}, we try to visualize the Information Bottleneck latent space. For the sake of visualization, we train an IB model using \texttt{REMEDI} with 2-d latent space on MNIST.
Training using Algorithm~\ref{REMEDI} involves learning the KNIFE base distribution and the corrections applied by the \texttt{REMEDI}. Therefore, we analyze how each of the two components captures the marginal distribution induced by the IB latent space.
\begin{figure}[H]
    \centering
	\centering
	\begin{subfigure}{0.15\textwidth}
	\centering
	\includegraphics[width=\linewidth]{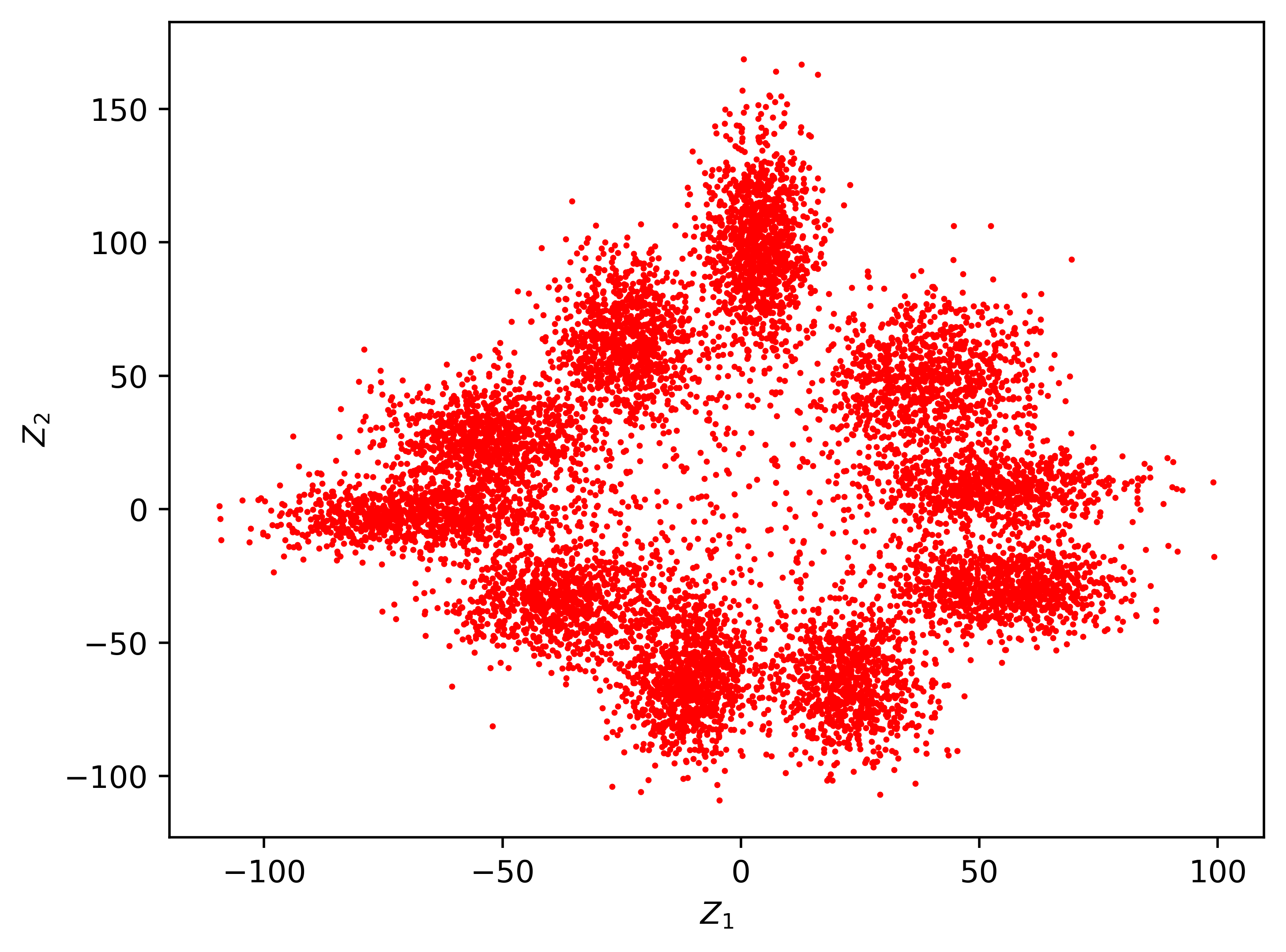}
	\caption{Encoder samples}
        \label{fig:encsamples}
	\end{subfigure}
	\begin{subfigure}{0.16\textwidth}
	\centering
	\includegraphics[width=\linewidth]{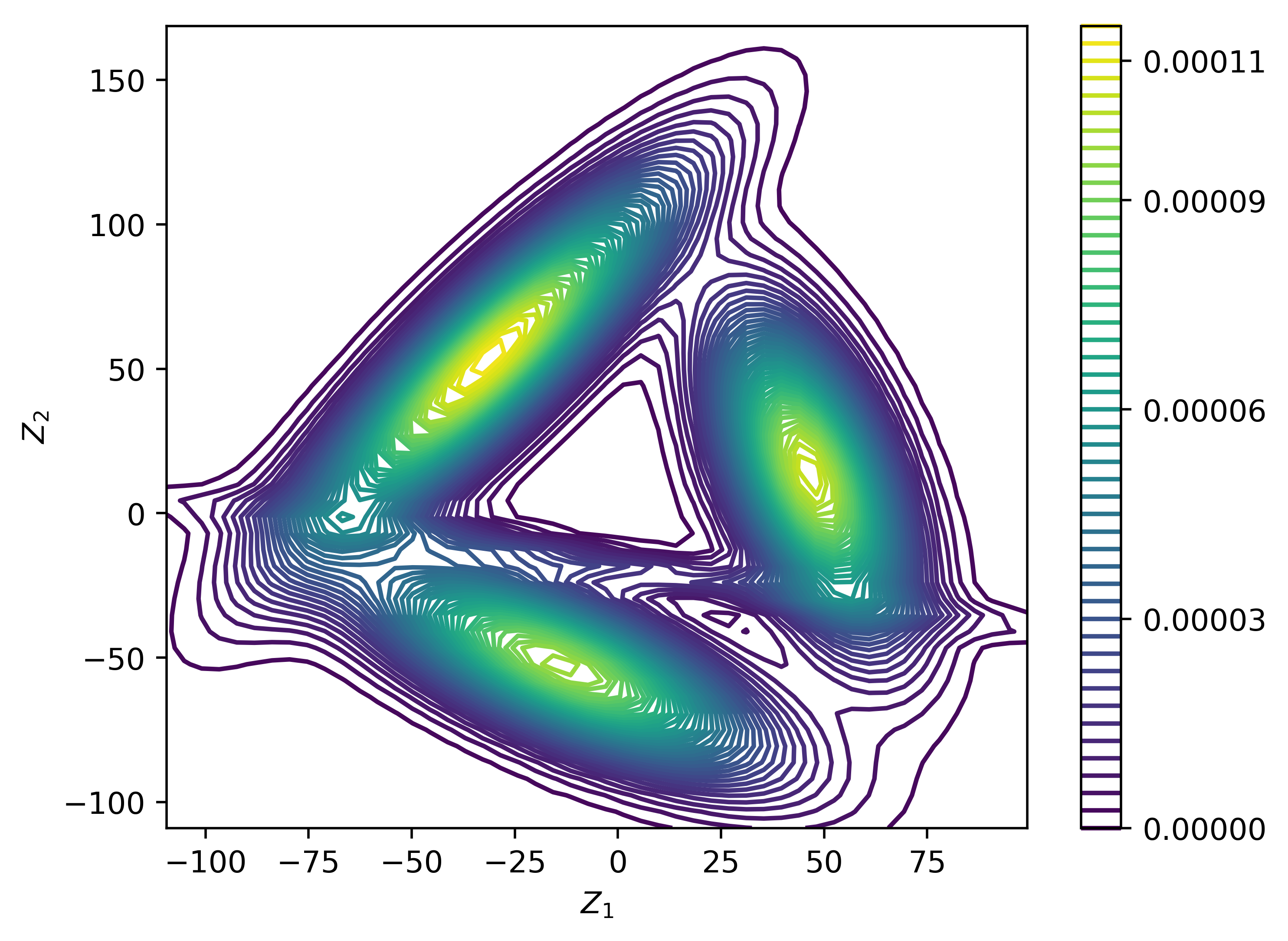}
	\caption{KNIFE}
        \label{fig:knifedensity}
	\end{subfigure}
	\begin{subfigure}{0.16\textwidth}
	\centering
	\includegraphics[width=\linewidth]{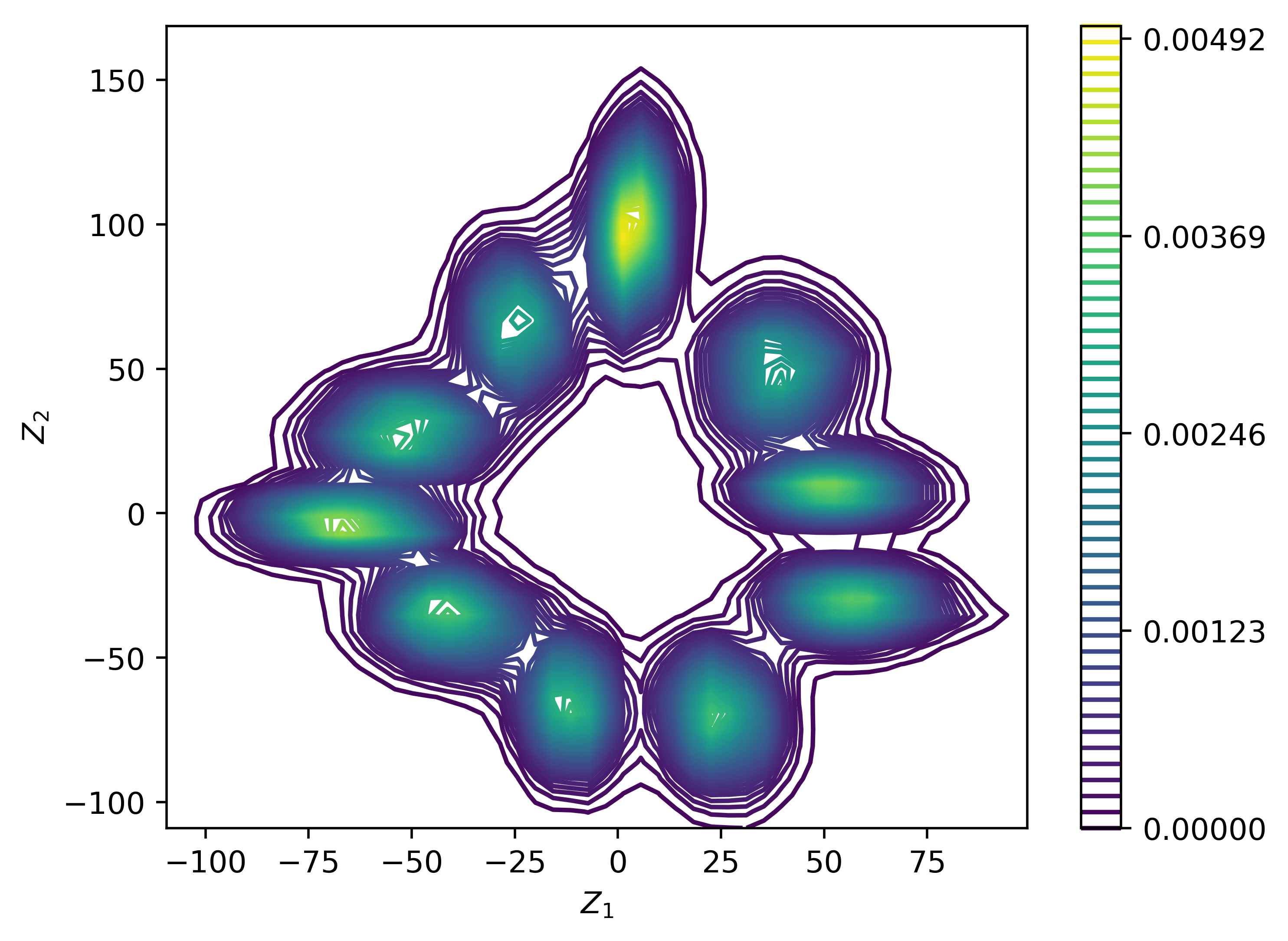}
	\caption{\texttt{REMEDI}}
        \label{fig:REMEDI}
	\end{subfigure}
    \caption{\texttt{REMEDI} marginal distribution of 2-d latent space on MNIST.}
    \label{fig:latentspace}
\end{figure}
We observe that the samples from the marginal distribution of $Z$ exhibit a clustering pattern in Fig.~\ref{fig:encsamples} where we can identify 10 cluster components. These clusters represent the 10 MNIST digits \cite{alemi2016deep}. The KNIFE base distribution, although having 10 mixture components, learns mixture distribution in Fig.~\ref{fig:knifedensity} with three identifiable modes. We note that KNIFE struggles with identifying clusters that are overlapping. In Fig.~\ref{fig:REMEDI}, we observe that applying the \texttt{REMEDI} based corrections helped improve the density estimation significantly. Accurate marginal density estimation in Fig.~\ref{fig:REMEDI} implies the latent space is properly regularized. In Appendix~\ref{sec:latentspace_evolution}, we perform this analysis throughout the training process from which the conclusions were similar. A similar analysis is also presented for CIFAR-10 in the Appendix~\ref{sec:latentcifar10}. This perhaps explains the improvement in accuracy shown by \texttt{REMEDI} over KNIFE in classification tasks on MNIST and CIFAR-10.

\subsection{Generative Models}
One useful by-product of fitting a neural network $T$ is that we have implicitly defined a density $\tilde{p}(x) \coloneq \frac{1}{C} q(x) e^{T(x)}$, with an unknown $C$, that approximates $p(x)$. This can be used in two common strategies for Monte Carlo sampling: rejection sampling and sampling based on Langevin dynamics. These two strategies and corresponding experiments are briefly explained in the following subsections.

\subsubsection{Rejection sampling}

Rejection sampling to sample from $\tilde p$, using the density $q$ of $\Q$ for comparisons, amounts to the following procedure. First, draw a sample $X$ from $\Q$. This sample is then accepted or rejected based on comparing $\tilde p (X)$ and $q(X)$. More concretely, if there is a constant $\tilde C \in (1, \infty)$ such that $\tilde p(x) / \tilde C q(x) \leq 1 $ for all $x$ such that $\tilde p(x) > 0$, then $X$ is accepted with probability  $\tilde p (X) / \tilde C q(X)$. 

From the specific form of $\tilde p$, we see that rejection sampling here requires a constant $\tilde{C}$ such that $e^{T(x)} \leq \tilde{C}$. For such a constant, a sample $X$ from $\mathbb{Q}$ is accepted with probability $\phi (X)$, where $\phi(x) \coloneqq \frac{e^{T(x)}}{\tilde{C}}$.

Revisiting the two moons dataset, using the learned $T$, rejection sampling is performed by taking $\tilde{C}$ as the maximum of $e^T$ over the dataset, plus a small margin.
In the experimental run illustrated in Fig.~\ref{fig:two-moons-rejection-sampling}, $10000$ samples from $\Q$ were used, out of which $2012$ samples were accepted. 

\begin{figure}[H]
	\centering
	\begin{subfigure}{0.2375\textwidth}
    	\centering
        \includegraphics[width=\linewidth]{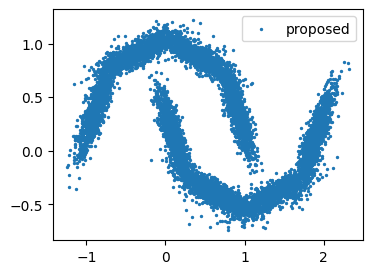}
    \end{subfigure}
	\begin{subfigure}{0.2375\textwidth}
    	\centering
        \includegraphics[width=\linewidth]{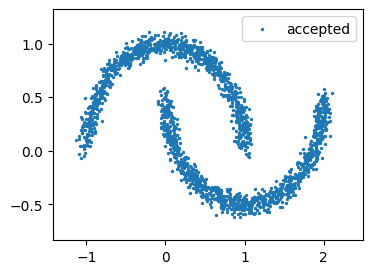}
	\end{subfigure}
    \caption{Left: $10000$ proposals from $\Q$. Right: $2012$ accepted samples.}
    \label{fig:two-moons-rejection-sampling}
\end{figure}

\subsubsection{Stochastic Differential Equations}

\begin{figure}[H]
    \hspace{-0.6em}
    \includegraphics[width=0.5\textwidth]{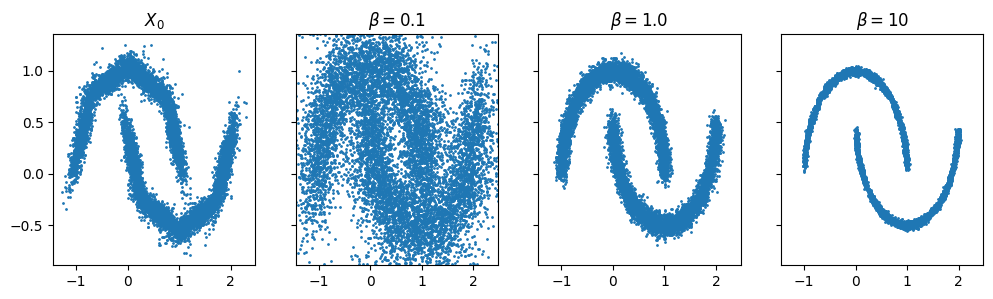}
    \caption{$\Q$-samples $X_0$ (leftmost) and $X_{t_H}$ after simulating Eq.~\eqref{eq:langevin} with different $\beta$.}
    \label{fig:two-moons-sde-samples}
\end{figure}
An alternative to rejection sampling, that does not require sampling a large number of random variables from $\mathbb{Q}$, is to use Langevin dynamics based on $q$ and $T$ to sample from $\tilde p$. This strategy is based on the stochastic differential equation
\begin{equation}\label{eq:langevin}
    dX_t = - \nabla V (X_t) dt + \sqrt{2 \beta^{-1}} dW_t, \quad X_0 = x_0,
\end{equation}
where the drift is defined by $V(x) = - \left(\log q (x) + T(x) \right)$; $\beta \in (0, \infty)$ is referred to as the inverse temperature. It is straightforward to show that the invariant distribution of $X = \{ X_t \} _{t \geq 0}$ is given by
\begin{equation}\label{eq:gibbs}
    \frac{1}{Z} e^{- \beta V(x)},
\end{equation}
where $Z$ is the normalizing constant. Note that for the choice $\beta = 1$, the exponent is precisely $- V(x) = \left(\log q (x) + T(x) \right)$ and the density of the invariant distribution corresponds precisely to $\tilde{p} $. Based on this, the Langevin dynamics Eq.~\eqref{eq:langevin} can be used to obtain samples from $\tilde p$. 

Fig.~\ref{fig:two-moons-sde-samples} illustrates the outcome of the following experiment: Taking $X_0 \sim \Q$, we set the time horizon to $ t_H = 0.1$, and simulate Eq.~\eqref{eq:langevin} using the Euler-Maruyama method with a discretization parameter $\Delta t = 0.001$ (i.e., 100 time steps).

\section{Conclusion / Future work}
In this paper, we introduce \texttt{REMEDI}, a mixture model corrective transformation approach that combines recent plug-in based entropy estimators with Donsker-Varadhan based objectives, to improve the estimation of information theoretic quantities.
We demonstrate the applicability of our approach to a variety of tasks ranging from entropy estimation, supervised learning, to generative models. We theoretically show the consistency of the \texttt{REMEDI} estimator under non-compactly supported data distributions, which is required by our framework. 
Using a range of benchmark datasets, we show that \texttt{REMEDI} outperforms the state-of-the-art entropy estimation approaches.
We discuss the application of \texttt{REMEDI} to the Information Bottleneck. 
We show that using this approach improves the performance of Information Bottleneck in classification tasks on MNIST, CIFAR-10, and ImageNet compared to the current state-of-the-art.
In addition, we show proof-of-concept that the \texttt{REMEDI} framework can be applied to generative tasks, using approaches such as rejection sampling and Langevin diffusion sampling.

It remains to be seen how this method performs in very high dimension, for example, what results can be acquired when performing entropy estimation on image data such as MNIST.
When using a uniform base distribution, \cite{park2021deep} finds that their method has some interpretable results on MNIST, but it seems that it only weakly learns to distinguish between in-distribution and out-of-distribution data.
Most likely, the sample efficiency is much too low to learn much about the dataset.
In these settings, instead using Gaussian mixture models such as KNIFE, copulas, or even normalizing flows may lead to base distributions that provide samples close to the low-dimensional data manifold, while also having tractable density functions.
Also, following the successful application of the Langevin diffusion approach presented here, the connection to the recently popularized score-based diffusion models \cite{song2020score} should be explored.

\section*{Acknowledgements}

The computations were enabled by the Berzelius resource provided by the Knut and Alice Wallenberg Foundation at the National Supercomputer Centre, and the computational resources of the Argonne Leadership Computing Facility, which is a DOE Office of Science User Facility supported under Contract DE-AC02-06CH11357, Laboratory Computing Resource Center (LCRC) at the Argonne National Laboratory. AS and SM were supported by the U.S. Department of Energy, Office of Science, Office of Fusion Energy Sciences, under contract DE-AC02-06CH11357 and Advanced Scientific Computing Research, through the SciDAC-RAPIDS2 institute under Contract DE-AC02-06CH11357. VN and PN were supported by the Wallenberg AI, Autonomous Systems and Software Program (WASP) funded by the Knut and Alice Wallenberg Foundation. PN was also supporrted in part by the Swedish Research Council (VR-2018-07050, VR-2023-03484).

\section*{Impact Statement}

This paper presents work whose goal is to advance the field of Machine Learning.
There are many potential societal consequences of our work, none which we feel must be specifically highlighted here.

\bibliography{ms}
\bibliographystyle{icml2024}

\newpage
\appendix
\onecolumn
\section{Consistency of \texttt{REMEDI}}\label{sec:consistency}

\noindent Techniques such as Parzen-window estimation and KNIFE are weaker than \texttt{REMEDI} but are known to be consistent under rather weak assumptions \cite{ahmad1976nonparametric}.
This requires letting bandwidths go to zero and the number of mixture components to infinity, sometimes quickly becoming intractable in even moderate dimensions \cite{wasserman2004all}.
To justify the method of letting ReLU-networks take some of the load and doing away with the infinitely many components, it is important to understand whether such a consistency result can be restated with the growing complexity of the function class of networks taking its place.

In this section, we will show that REMEDI is consistent, in the sense that it can approximate the relative entropy arbitrarily well with many samples.
\cite{belghazi2018mutual} show consistency of their mutual information estimator, under the (strong) assumption that the data and distributions are compactly supported.
Such an assumption is undesirable in our case, since the data $\P$ may have infinite tails (e.g. two moons) and the base distributions $\Q$ are Gaussian (mixtures). 

It is possible to adapt the results of \cite{belghazi2018mutual}, to our entropy estimation setting.
However, the generalization to non-compact supports is non-trivial, and we provide it here for a specific class of neural networks, very similar to the ones that we have used.

Assume that $\P$ and $\Q$ are probability measures on $\RR^d$ with the Borel $\sigma$-algebra $\B(\RR^d)$, such that $\P \ll \Q$.
Recall that the \textit{relative entropy}, or KL-divergence, between $\P$ and $\Q$ is defined as

\begin{equation}
    R(\P \parallel \Q) \coloneq \E^{\P}\left[\log \frac{d\P}{d\Q}\right].
\end{equation}

\noindent For a measurable function $T: \RR^d \to \RR$, define the functional

\begin{equation}\label{eq:R-functional}
    \DV(T) \coloneq \E^{\P}[T] - \log\E^{\Q}[e^T],
\end{equation}

\noindent and with a slight abuse of notation let, for a family $\F$ of functions $T: \RR^d \to \RR$,

\begin{equation}\label{eq:R-hat-functional}
    \DV(\F) \coloneq \sup_{T \in \F} \DV(T) = \sup_{T \in \F} (\E^{\P}[T] - \log\E^{\Q}[e^T]).
\end{equation}

\noindent The motivation behind these definitions is the central Donsker-Varadhan representation of the relative entropy, here restated in terms of $\DV$.

\begin{proposition}[Donsker-Varadhan]
    Let $\P \ll \Q$, then

    \begin{equation}
        R(\P \parallel \Q) = \sup_{T \in C_b} \DV(T) = \DV({C_b}),
    \end{equation}

    \noindent where $C_b$ is the set of bounded continuous functions.
    The supremum may also be taken over bounded measurable functions.
    The supremum is attained at $T = \log \frac{d\P}{d\Q}$, which may not be bounded or continuous.
\end{proposition}

\noindent Let $\P_n$ and $\Q_m$ be the empirical measures over $n$ and $m$ independent samples from $\P$ and $\Q$, respectively.
We define, for $n, m \in \NN$, the empirical version of \eqref{eq:R-hat-functional}, based on $\P_n, \Q_m$, as

\begin{equation}\label{eq:R-hat-functional-empirical}
    \DVnm(\F) \coloneq \sup_{T \in \F} (\E^{\P_n}[T] - \log\E^{\Q_m}[e^T]).
\end{equation}

Note that in contrast to $R(\P \parallel \Q)$ and $\DV(\F)$, the quantity $\DVnm(\F)$ is random, and it is not clear in which fashion it converges to $\DV(\F)$, if at all.
This question falls under the field of empirical process theory \cite{van2000empirical}.
A few common results from it will provide conditions on the class $\F$ such that $\DVnm(\F)$ converges to $\DV(\F)$, almost surely.
We will give an overview of these and show how a class of ReLU networks fulfills them, while also being expressive enough for Eq.~\eqref{eq:R-hat-functional} to approximate $R(\P \parallel \Q)$ arbitrarily well.

\subsection{Assumptions}\label{sec:assumptions}

As previously stated, we relax the assumption of \cite{belghazi2018mutual} about compactness of the data space $\mathcal{X} \subseteq \RR^d$.
Thus we simply assume that $\P$ and $\Q$ are probability measures on $\RR^d$, with $\P \ll \Q$ and $R(\P \parallel \Q) < \infty$.
We will require finite first moments, i.e., for $X$ the identity mapping on $\RR^d$, we have (where $||.||_p$ denotes the $p$-norm)

\begin{equation}
    \E^{\P}[||X||_2], \E^{\Q}[||X||_2] < \infty.
\end{equation}

\noindent We also assume that $\Q$ is sub-Gaussian (see definition below), which holds for Gaussian mixtures, and that our parameter space $\Theta$ is a compact subset of $\RR^N$ for some $N$.

With the compactness assumption on $\Theta$, it will follow that there is a common global Lipschitz constant $L_\Theta$ (with respect to the Euclidean norm $||.||_2$) and a common bound $A_\Theta$ at zero for the neural network functions $\{T_\theta: \RR^d \to \RR\}_{\theta \in \Theta}$, when using for example ReLU activation functions.
With these assumptions, the proof will follow from arguments from empirical process theory, taken from \cite{van2000empirical}, and a deeper look at the class of ReLU networks, specifically those with two hidden layers.

Below, we state the definitions of sub-Gaussian and sub-exponential random variables, of which especially the latter is of crucial importance to the proof.
We follow the presentation in \cite{vershynin2018high}.
Note that these definitions do not require centeredness.

\begin{definition}\label{def:sub-Gaussian}
    We say that the random variable $X$ in $\RR$ is \textbf{sub-Gaussian} if there is a constant $K_1$ such that

    \begin{equation}
        \P(|X| \geq t) \leq e^{-\frac{t^2}{K_1^2}}, \quad \forall t \geq 0.
    \end{equation}

    Further, we say that the random vector $X \in \RR^d$ is sub-Gaussian if $u^T X$ is sub-Gaussian for all $u \in \RR^d$ with $||u|| = 1$.
    Additionally the probability measure $\P$ on $\RR^d$ is sub-Gaussian if $X = \mathrm{id}_{\RR^d}$ is sub-Gaussian under $\P$.
\end{definition}

\begin{definition}\label{def:sub-exponential}
    We say that the random variable $X$ in $\RR$ is \textbf{sub-exponential} if there is a constant $K_1$ such that

    \begin{equation}
        \P(|X| \geq t) \leq e^{-\frac{t}{K_1}}, \quad \forall t \geq 0.
    \end{equation}
\end{definition}

The following properties of sub-Gaussian and sub-exponential distributions will become necessary.

\begin{proposition}\label{pro:sub}
    The following facts hold for sub-Gaussian and sub-exponential variables.

    \begin{enumerate}
        \item The components of a sub-Gaussian vector are sub-Gaussian.\label{pro:sub-0}
        \item A sum of sub-exponential random variables is sub-exponential.\label{pro:sub-1}
        \item $X$ is sub-Gaussian if and only if $X^2$ is sub-exponential.\label{pro:sub-2}
        \item If $X$ is a sub-Gaussian random vector in $\RR^d$, then the random variable $||X||_2$, its Euclidean norm, is sub-Gaussian.\label{pro:sub-3} 
        \item If $X$ is a centered ($\E[X] = 0$), sub-Gaussian random variable on $\RR$, then there exists a \textbf{variance proxy} $\sigma^2 \geq 0$ such that $\E[e^{tX}] \leq e^{\frac{\sigma^2 t^2}{2}}$ for all $t \in \RR$. We say then say that $X \sim \subg(\sigma^2)$.\label{pro:sub-4}
        \item If $X_1 \sim \subg(\sigma_1^2)$ and $X_2 \sim \subg(\sigma_2^2)$ are independent, then $X_1 + X_2 \sim \subg(\sigma_1^2 + \sigma_2^2)$.\label{pro:sub-5}
        \item If $X \sim \subg(\sigma^2)$ and $a \in \RR$, then $aX \sim \subg(a^2 \sigma^2)$.\label{pro:sub-6}
        \item If $X \sim \subg(\sigma^2)$, then for any $a \geq 0$ we have $\P(X \geq a) \leq \exp\left(-\frac{a^2}{2 \sigma^2}\right)$.\label{pro:sub-7}
    \end{enumerate}
\end{proposition}

\begin{proof}
    Property \ref{pro:sub-0} follows immediately from the definition by taking $u = (1, 0, 0, \ldots), (0, 1, 0, \ldots)$ etc.
    Properties \ref{pro:sub-1}, \ref{pro:sub-2} and \ref{pro:sub-4} are standard, see \cite{vershynin2018high}.
    Property \ref{pro:sub-3} follows from Properties \ref{pro:sub-0}, \ref{pro:sub-1} and \ref{pro:sub-2} since $||X||_2^2 = \sum_{i=1}^d X_i^2$ is the sum of sub-exponential random variables and thus sub-exponential.
    The following computation shows Property \ref{pro:sub-5}, while Property \ref{pro:sub-6} is immediate.

    \begin{equation}
        \E[e^{t(X_1 + X_2)}] = \E[e^{tX_1}]\E[e^{tX_2}] \leq e^{\frac{\sigma_1^2 t^2}{2}} e^{\frac{\sigma_2^2 t^2}{2}} = e^{\frac{(\sigma_1^2 + \sigma_2^2) t^2}{2}}
    \end{equation}

    To obtain Property \ref{pro:sub-7}, we minimize a Chernoff bound; we have, for any $t \in \RR$,

    \begin{equation}\label{eq:concentration-chernoff}
        \P(X \geq a) = \E[\indic_{X \geq a}] \leq \E\left[\indic_{X \geq a} \frac{e^{tX}}{e^{ta}}\right] \leq \E\left[\indic_{X \geq a} \frac{e^{tX}}{e^{ta}}\right] = \E\left[\frac{e^{tX}}{e^{ta}}\right] = e^{-ta} \E\left[e^{tX}\right] \leq e^{-ta + \frac{\sigma^2 t^2}{2}}.
    \end{equation}

    Minimizing this over $t$ yields $t = \frac{a}{\sigma^2}$, which upon reinsertion into Equation Eq.~\eqref{eq:concentration-chernoff} gives

    \begin{equation}
        \P(X \geq a) \leq e^{-ta + \frac{\sigma^2 t^2}{2}}|_{t = \frac{a}{\sigma^2}} = e^{-\frac{a^2}{2 \sigma^2}}.
    \end{equation}
\end{proof}

\subsection{Uniform laws of large numbers}

An overview over necessary results from \cite{van2000empirical} is given below.
Consider a class $\G$ of functions $g: \mathcal{X} \to \RR$.
Let the \emph{empirical measure} be $\P_n \coloneq \frac{1}{n} \sum_{i=1}^n \delta_{X_i}$, where $X_i$ are i.i.d. draws from $\P$.
We are interested in asserting the convergence of
\begin{equation}
    \E^{\P_n}[g] = \frac{1}{n} \sum_{i=1}^n g(X_i)
\end{equation}
\noindent toward $\E^{\P}[g]$ for all $g$ in $\G$.
This is equivalent to the convergence of the "worst" $g$ with respect to the data, represented by $\P_n$, i.e. of $\sup_{g \in \G} |\E^{\P_n}[g] - \E^{\P}[g]|$ toward zero.
This is a \emph{uniform law of large numbers} (ULLN), which is called \emph{strong} if the convergence happens almost surely.

Let $p \geq 1$ be a norm exponent; we will be interested in $p=1$. 
The bracketing number and bracketing entropy are defined as follows.

\begin{definition}
    Let $\delta > 0$ and $\P$ be a probability measure on $\mathcal{X}$.
    Assume that there exists a collection of pairs of functions $\{[g^L_j, g^U_j]\}_{j=1}^N$ in $\G$ such that $||g^U_j - g^L_j||_p \leq \delta$ and for each $g \in \G$ there exist $j$ such that $g^L_j \leq g \leq g^U_j$.
    The smallest such $N$ (or $\infty$) is the \textbf{bracketing number} $N_{p,B}(\delta, \G, \P)$, and the \textbf{bracketing entropy} is $H_{p,B}(\delta, \G, \P) = \log N_{p,B}(\delta, \G, \P)$.
\end{definition}

The following law of large numbers will be used to prove the consistency.
Note that the proof given below differs slightly from that provided in \cite{van2000empirical}.

\begin{lemma}[Lemma 3.1 in van de Geer]\label{lem:finite-entropy-bracketing-ulln}
    Assume that $H_{1, B}(\delta, \G, \P) < \infty$ for all $\delta > 0$.
    Then $\G$ satisfies the strong uniform law of large numbers (ULLN):
    if $\{X_i\}_{i=1}^n$ are i.i.d. samples from $\P$, then
    \begin{equation}\label{eq:ulln}
        \sup_{g \in \G} \left|\E^{\P_n}[g] - \E^{\P}[g]\right| = \sup_{g \in \G} \left|\frac{1}{n} \sum_{i=1}^n g(X_i) - \E^{\P}[g]\right| \xrightarrow{a.s.} 0.
    \end{equation}
\end{lemma}

\begin{proof}
    Use the empirical process notation $\P f = \E^{\P}[f(X)]$ for expectation with respect to $\P$.
    Then Eq.~\eqref{eq:ulln} is $ \sup_{g \in \G} \left|\P_n g - \P g\right| \xrightarrow{a.s.} 0$.
    Take $\delta > 0$.
    Then we have the above finite collection $\{[g^L_j, g^U_j]\}_{j=1}^{N_\delta}$, and for each $g \in \G$ there exists $j = j(g)$ such that $g^L_j \leq g \leq g^U_j$.
    Then it holds for any $g$ that $|\P_n g - \P g| \leq |\P_n g^U_{j(g)} - \P g| \lor |\P_n g^L_{j(g)} - \P g| = \max_{B \in \{L, U\}} |\P_n g^B_{j(g)} - \P g|$, and thus we have
    \begin{equation}
    \begin{split}
        \limsup_{n \to \infty} \sup_{g \in \G} \left|\P_n g - \P g\right| &\leq \limsup_{n \to \infty} \sup_{g \in \G} \max_{B \in \{L, U\}} \left|\P_n g^B_{j(g)} - \P g\right| \\ 
        &= \limsup_{n \to \infty} \sup_{g \in \G} \max_{B \in \{L, U\}} \left| (\P_n - \P) g^B_{j(g)} + \P (g^B_{j(g)} - g) \right| \\
        &\leq \limsup_{n \to \infty} \sup_{g \in \G} \max_{B \in \{L, U\}} \left| (\P_n - \P) g^B_{j(g)} \right| + \delta \\
        &= \delta + \limsup_{n \to \infty} \max_{j \in \{1, \dots, N_\delta\}} \max_{B \in \{L, U\}} \left| (\P_n - \P) g^B_{j} \right| \\
        &= \delta \text{ a.s.},
    \end{split}
    \end{equation}

    \noindent where the last equality follows from the strong law of large numbers applied to each $g^U_j(X_i)$. 
    Taking intersection of the events where the SLLN holds for each pair $(j, B) \in \{1, \ldots, N_\delta\} \times \{L, U\}$, yields an event of probability one.
    Since they are finitely many pairs, a simple argument shows that $(\max_{j \in \{1, \dots, N_\delta\}} | (\P_n - \P) g^U_{j} |)(\omega)$ converges for all $\omega$ in this intersection.

    Since $\delta$ is arbitrarily small, we have that $\limsup_{n \to \infty} \sup_{g \in \G} \left|\P_n g - \P g\right| = 0$ almost surely, which is the claim.
\end{proof}

Now we need to prove that $\{T_\theta: \theta \in \Theta\}$ satisfies the bracketing condition, $H_{1, B}(\delta, \G, \P) < \infty$ for all $\delta > 0$, with respect to a sub-Gaussian $\P$.

\begin{definition}
    The \textbf{envelope} of $\G$ is the function
    \begin{equation}
        G(x) = \sup_{g \in \G} |g(x)|.
    \end{equation}

    \noindent We say that $\G$ satisfies the \textbf{envelope condition} if $G \in L_1(\P)$.
\end{definition}

Recall that $\F_\Theta$ shares a common Lipschitz constant $L_\Theta$ and bound $A_\Theta$ at zero.

\begin{lemma}\label{lem:envelope-condition}
    $\F_\Theta = \{T_\theta : \theta \in \Theta\}$ satisfies the envelope condition with respect to $\P$.
\end{lemma}

\begin{proof}
    Note first that, for any $x$,
    
    \begin{equation}
    \begin{split}
        G(x) &= \sup_{\theta \in \Theta} |T_\theta(x)| \\ 
        &= \sup_{\theta \in \Theta} |T_\theta(0) + (T_\theta(x) - T_\theta(0))| \\
        &\leq \sup_{\theta \in \Theta} |T_\theta(0)| + \sup_{\theta \in \Theta} |T_\theta(x) - T_\theta(0)| \\
        &\leq A_\Theta + L_\Theta||x||_2.
    \end{split}
    \end{equation}
    
    \noindent Thus we obtain, by finite first moment of $\P$,

    \begin{equation}
        \E^{\P}\left[G\right] \leq \E^{\P}\left[A_\Theta + L_\Theta||X||_2\right] = A_\Theta + L_\Theta\E^{\P}\left[||X||_2\right] < \infty.
    \end{equation}
\end{proof}

\begin{lemma}\label{lem:envelope-condition-eT-Q}
    $\F^{\exp}_\Theta = \{e^{T_\theta} : \theta \in \Theta\}$ satisfies the envelope condition with respect to $\Q$.
\end{lemma}

\begin{proof}
    As noted in proof of Lemma~\ref{lem:envelope-condition}, $\sup_{\theta \in \Theta} |T_\theta(x)| \leq A_\Theta + L_\Theta||x||_2$.
    By monotonicity of the supremum and the map $x \mapsto e^x$, we have
    \begin{equation}
        G^{\exp}(x) = \sup_{\theta \in \Theta} |e^{T_\theta(x)}| \leq \sup_{\theta \in \Theta} e^{|T_\theta(x)|} = e^{\sup_{\theta \in \Theta} |T_\theta(x)|} \leq e^{A_\Theta + L_\Theta||x||_2}.
    \end{equation}

    \noindent Thus we obtain,

    \begin{equation}
        \E^{\Q}[G^{\exp}] \leq \E^{\Q}[e^{A_\Theta + L_\Theta||X||_2}] = e^{A_\Theta}\E^{\Q}[e^{L_\Theta||X||_2}],
    \end{equation}

    \noindent which by sub-Gaussianity of $\Q$ is finite.
\end{proof}

\begin{lemma}[Lemma 3.10 in van de Geer]\label{lem:entropy-bracketing-compact-theta}
    Assume that $\G = \{g_\theta : \theta \in \Theta\}$ satisfies the envelope condition with respect to $\P$.
    Assume also that $(\Theta, d)$ is a compact metric space, and that $\theta \mapsto g_\theta(x)$ is continuous for $\P$-almost all $x \in \mathcal{X}$.
    Then for any $\delta > 0$,

    \begin{equation}
        H_{1, B}(\delta, \G, \P) < \infty.
    \end{equation}
\end{lemma}

\begin{proof}
    See \cite{van2000empirical}.
\end{proof}

\noindent We now show that $\{T_\theta : \theta \in \Theta\}$ satisfies a \emph{strong uniform law of large numbers} (ULLN).

\begin{theorem}\label{thm:theta-uniform-convergence-non-compact}
    Suppose each element $T_\theta$ of the family of neural network functions $\F_{\Theta} = \{T_{\theta}: \mathcal{X} \to \RR\}_{\theta \in \Theta}$ depends continuously on $\theta$, for any fixed $x \in \mathcal{X}$. 
    With the standing assumptions on $\P$, $\Q$, and $\F_\Theta$ (see Section \ref{sec:assumptions}),
    we have that

    \begin{equation}
        \DVnm(\F_{\Theta}) \xrightarrow{a.s.} \DV(\F_{\Theta}),
    \end{equation}

    \noindent as $n, m \to \infty$.
\end{theorem}

\begin{proof}
    Using the identity $|\sup f - \sup g| \leq \sup |f - g|$, we have 
    \begin{equation}\label{eq:emp-gen-error}
    \begin{split}
        |\DVnm(\F_\Theta) - \DV(\F_\Theta)| &= \left|\sup_{T \in \F_\Theta} (\E^{\P_n}[T] - \log\E^{\Q_m}[e^T]) - \sup_{T \in \F_\Theta} (\E^{\P}[T] - \log\E^{\Q}[e^T])\right| \\
        &\leq \sup_{T \in \F_\Theta} \left|(\E^{\P_n}[T] - \log\E^{\Q_m}[e^T]) - (\E^{\P}[T] - \log\E^{\Q}[e^T])\right| \\
        &\leq \sup_{T \in \F_\Theta} \left|\E^{\P_n}[T] - \E^{\P}[T]\right| + \sup_{T \in \F_\Theta} \left|\log\E^{\Q_m}[e^T] - \log\E^{\Q}[e^T]\right| \\
        &= \sup_{\theta \in \Theta} \left|\E^{\P_n}[T_\theta] - \E^{\P}[T_\theta]\right| + \sup_{\theta \in \Theta} \left|\log\E^{\Q_m}[e^{T_\theta}] - \log\E^{\Q}[e^{T_\theta}]\right|. \\
    \end{split}
    \end{equation}
    
    By \Cref{lem:envelope-condition}, we have that $\F_\Theta$ satisfies the envelope condition with respect to $\P$.
    Thus by \Cref{lem:entropy-bracketing-compact-theta} and \ref{lem:finite-entropy-bracketing-ulln}, we have that $\F_\Theta$ satisfies the strong ULLN.
    Hence we have for the first term of \Cref{eq:emp-gen-error},

    \begin{equation}\label{eq:Pn-convergence}
        \sup_{\theta \in \Theta} \left|\E^{\P_n}[T_\theta] - \E^{\P}[T_\theta]\right| \xrightarrow{a.s.} 0.
    \end{equation}

    The second term of \Cref{eq:emp-gen-error} is more difficult, because the logarithm does not have a Lipschitz constant on $(0, \infty)$.
    To obtain one, we need to know that $\E^{\Q_m}[e^{T_\theta}]$ and $\E^{\Q}[e^{T_\theta}]$ can be suitably bounded from below.
    Let again $L_\Theta$ be the common Lipschitz constant of $\F_\Theta$ and $A_\Theta \coloneq \sup_{\theta \in \Theta} |T_\theta(0)|$. 
    Note first that 

    \begin{equation}
        \E^{\Q}[e^{T_\theta}] \geq \E^{\Q}[e^{-A_\Theta - L_\Theta ||X||_2}] > 0.
    \end{equation}
    
    \noindent Let $\mu_{2}^{\Q} \coloneq \E^{\Q}\left[||X||_2\right]$.
    For $b \coloneq \exp(-A_\Theta - L_\Theta (\mu_{2}^{\Q} + 1))$, we have on $\Theta$,

    \begin{equation}\label{eq:eT-subgaussian-hard-bound-1}
    \begin{split}
        \Prob\left(\E^{\Q_m}[e^{T_\theta}] < b, \forall \theta \in \Theta \right) &\leq \Prob\left(\E^{\Q_m}\left[e^{-A_\Theta - L_\Theta ||X||_2}\right] < b \right) \\
        \text{(Jensen's)} &\leq \Prob\left(\exp(\E^{\Q_m}\left[-A_\Theta - L_\Theta ||X||_2\right]) < b \right) \\
        &= \Prob\left(\E^{\Q_m}\left[-A_\Theta - L_\Theta ||X||_2\right] < \log b \right) \\
        &= \Prob\left(\E^{\Q_m}\left[||X||_2\right] > - \frac{A_\Theta + \log b}{L_\Theta} \right) \\
        &= \Prob\left(\E^{\Q_m}\left[||X||_2 - \mu_{2}^{\Q}\right] > - \frac{A_\Theta + \log b}{L_\Theta} - \mu_{2}^{\Q} \right) \\
        &= \Prob\left(\E^{\Q_m}\left[||X||_2 - \mu_{2}^{\Q}\right] > 1 \right).
    \end{split}
    \end{equation}

    \noindent Now, notice that $||X||_2 - \mu_{2}^{\Q}$ is centered sub-Gaussian under $\Q$, with some variance proxy $(\sigma^2)_2^{\Q}$.
    Then it holds that 

    \begin{equation}
        \E^{\Q_m}\left[||X||_2 - \mu_{2}^{\Q}\right] = \frac{1}{m} \sum_{i = 1}^m (||X||_2 - \mu_{2}^{\Q}) \sim \subg\left(\frac{(\sigma^2)_2^{\Q}}{m}\right).
    \end{equation}

    \noindent Returning to \Cref{eq:eT-subgaussian-hard-bound-1}, we get

    \begin{equation}\label{eq:eT-subgaussian-hard-bound-2}
    \begin{split}
        \Prob\left(\E^{\Q_m}\left[e^{T_\theta}\right] < b, \forall \theta \in \Theta \right) &\leq \Prob\left(\E^{\Q_m}\left[||X||_2 - \mu_{2}^{\Q}\right] > 1 \right) \\
        &\leq \exp\left(-1^2 / \left(2 \frac{(\sigma^2)_2^{\Q}}{m}\right)\right) \\
        &= \exp\left(-\frac{1}{2(\sigma^2)_2^{\Q}} m\right).
    \end{split}
    \end{equation}

    \noindent where in the second inequality we used the concentration inequality, Property \ref{pro:sub-7}, from Proposition \ref{pro:sub}, with $a = 1$.
    Since $\sum_{m=1}^\infty \exp\left(-\frac{1}{2(\sigma^2)_2^{\Q}} m\right) < \infty$, we get by the (first) Borel-Cantelli lemma that

    \begin{equation}
        \Prob\left(\E^{\Q_m}\left[e^{T_\theta}\right] < b, \forall \theta \in \Theta, \text{infinitely often}\right) = 0.
    \end{equation}

    \noindent In other words, on our probability space $(\Omega, \Sigma, \Prob)$, there is an event $\Omega' \in \Sigma$, with $\Prob(\Omega') = 1$ and a function $M: \Omega \to \NN$, such that $m \geq M(\omega) \implies \E^{\Q_m}\left[e^{T_\theta}\right](\omega) > b$ for $\omega \in \Omega'$.
    Let $\tilde{b} \coloneq b \land \E^{\Q}[e^{-A_\Theta - L_\Theta ||X||_2}]$.
    The function $\log|_{[\tilde{b}, \infty)}$ has a Lipschitz constant $1 / \tilde{b}$.
    On $\Omega'$, we have that for all $m \geq M(\omega)$,

    \begin{equation}\label{eq:lipschitz-bound-log-Qm}
        \sup_{\theta \in \Theta} \left|\log\E^{\Q_m}[e^{T_\theta}] - \log\E^{\Q}[e^{T_\theta}]\right|(\omega) \leq \frac{1}{\tilde{b}} \sup_{\theta \in \Theta} \left|\E^{\Q_m}[e^{T_\theta}] - \E^{\Q}[e^{T_\theta}]\right|(\omega).
    \end{equation}

    $\F^{\exp}_\Theta = \{e^{T_\theta}: \theta \in \Theta\}$ satisfies the envelope condition with respect to $\Q$, by \Cref{lem:envelope-condition-eT-Q}.
    By \Cref{lem:entropy-bracketing-compact-theta} and \ref{lem:finite-entropy-bracketing-ulln}, we have then that $\F^{\exp}_\Theta$ satisfies the strong ULLN.
    Thus we have that the left-hand side of \Cref{eq:lipschitz-bound-log-Qm} goes to zero.
    Then, since $\Omega'$ is a probability one event, we have for the second term of \Cref{eq:emp-gen-error},

    \begin{equation}\label{eq:Qn-convergence}
        \sup_{\theta \in \Theta} \left|\log\E^{\Q_m}[e^{T_\theta}] - \log\E^{\Q}[e^{T_\theta}]\right| \xrightarrow{a.s.} 0.
    \end{equation}

    \noindent Because the intersection of the events in \Cref{eq:Pn-convergence} and \Cref{eq:Qn-convergence} is a probability one event, we have

    \begin{equation}
        \sup_{\theta \in \Theta} \left|\E^{\P_n}[T_\theta] - \E^{\P}[T_\theta]\right| + \sup_{\theta \in \Theta} \left|\log\E^{\Q_m}[e^{T_\theta}] - \log\E^{\Q}[e^{T_\theta}]\right| \xrightarrow{a.s.} 0,
    \end{equation}

    \noindent so by \Cref{eq:emp-gen-error}, $\DVnm(\F_\Theta) \xrightarrow{a.s.} \DV(\F_\Theta)$.
\end{proof}

\subsection*{Universal approximation}

We will also need to show that the we can use the universal approximation theorem even when $\mathcal{X}$ is non-compact.
We begin by proving that, for a constant $M \in [0, \infty]$, we can truncate the output of a ReLU network with one hidden layer to $[-M, M]$, by adding another layer.

\begin{lemma}\label{lem:nn-truncation}
    Let $T_{\theta^1}: \RR^d \to \RR$ be a neural network with one hidden layer, with ReLU activation functions.
    Then, for $M \geq 0$ there exists a two hidden layer ReLU network $T_{\theta^2}$ which truncates $T_{\theta^1}$ to $[-M, M]$, i.e., 
    
    \begin{equation}
        T_{\theta^2}(x) = M \land (-M \lor T_{\theta^1}(x)), \quad \forall x \in \RR^d.
    \end{equation}
\end{lemma}

\begin{proof}
    The output of the one hidden layer neural network, with width $l$, can be written

    \begin{equation}
        T_{\theta^1}(x) = \sum_{i=1}^l \alpha_i (\beta_i^T x + \gamma_i)^+.
    \end{equation}

    Let the first hidden layer of $T_{\theta^2}$ be identical to the hidden layer in $T_{\theta^1}$.
    Let its second hidden layer consist of two nodes with weights $\alpha^{(1)}_i = \alpha^{(2)}_i = \alpha_i$, and biases $M$ and $-M$, respectively.
    Then the activations of the second hidden layer are
    \begin{equation}
    \begin{split}
        o^{(1)} &= \left(\sum_{i=1}^l \alpha_i (\beta_i^T x + \gamma_i)^+ + M\right)^+ = (T_\theta(x) - (-M))^+ \\
        o^{(2)} &= \left(\sum_{i=1}^l \alpha_i (\beta_i^T x + \gamma_i)^+ - M\right)^+ = (T_\theta(x) - M)^+,
    \end{split}
    \end{equation}

    \noindent as well as the bias node $o^{(0)} = 1$.
    Let the output layer consist of one node with weights $\lambda^{(0)} = -M$, $\lambda^{(1)} = 1$, and $\lambda^{(2)} = -1$.
    Then the output of the network is 

    \begin{equation}
    \begin{split}
        T_{\theta^2}(x) &= \lambda^{(0)} + \lambda^{(1)} o^{(1)} + \lambda^{(2)} o^{(2)} \\
        &= -M + (T_\theta(x) - (-M))^+ - (T_\theta(x) - M)^+ \\
        &= M \land (-M \lor T_{\theta^1}(x)).
    \end{split}
    \end{equation}
\end{proof}

\begin{lemma}\label{lem:approximation-non-compact}
    Let $\P$ be a probability measure on $\RR^d$.
    Let $f: \RR^d \to \RR$ be bounded and continuous.
    Then, for $\varepsilon > 0$, there exists a two hidden layer neural network $T_\theta: \RR^d \to \RR$ such that $\E^{\P}[|f - T_\theta|] < \varepsilon$.
    Further, if $f$ is bounded by $M' - 1$, then $T_\theta$ can be chosen to be bounded by $M'$.
\end{lemma}

\begin{proof}
    Let $M = \sup_{x \in R^d} |f(x)| + 1$.
    Take a compact set $C = [-K, K]^d \subseteq \mathcal{X}$ such that $\P(C^c) < \frac{\varepsilon}{3M}$.

    Now, by the universal approximation theorem \cite{hornik1989multilayer}, take a one hidden layer neural network $T_{\theta^1}$ such that $\sup_{x \in C} |T_{\theta^1}(x) - f(x)| < \frac{\varepsilon}{3}$.
    Then we have that 

    \begin{equation}
        \sup_{x \in C} |T_{\theta^1}(x)| \leq \sup_{x \in C} |T_{\theta^1}(x) - f(x)| + \sup_{x \in C} |f(x)| < \frac{\varepsilon}{3} + (M - 1) < M.
    \end{equation}

    Assuming that $T_{\theta^1}$ uses ReLU activation functions (which is compatible by simple extension of the result in \cite{hornik1989multilayer}), by Lemma~\ref{lem:nn-truncation} there is an alternative two hidden layer network $T_{\theta}$ such that $T_{\theta}(x) = M \land (-M \lor T_{\theta^1}(x))$.
    Note that $T_{\theta}(x) = T_{\theta^1}(x)$ for $x \in C$, since $T_{\theta^1}(x) \in [-M, M]$ for $x \in C$.

    Then we have that $\sup_{x \in C} |T_{\theta}(x) - f(x)| < \frac{\varepsilon}{3}$ and $|T_{\theta}(x)| \leq M$ for all $x \in \RR^d$.
    This gives us that

    \begin{equation}
    \begin{split}
        \E^{\P}[|f - T_\theta|] &= \E^{\P}[|f - T_\theta| \indic_{C}] + \E^{\P}[|f - T_\theta| \indic_{C^c}] \\
        &\leq \E^{\P}[|f - T_\theta| \indic_{C}] + \E^{\P}[|f| \indic_{C^c}] + \E^{\P}[|T_\theta| \indic_{C^c}] \\
        &\leq \E^{\P}[(\varepsilon / 3) \indic_{C}] + \E^{\P}[M \indic_{C^c}] + \E^{\P}[M \indic_{C^c}] \\
        &< \frac{\varepsilon}{3} + \frac{\varepsilon}{3} + \frac{\varepsilon}{3} \\
        &= \varepsilon.
    \end{split}
    \end{equation}

    \noindent The last statement of the lemma follows, since we can take $M = M'$.
\end{proof}

\begin{theorem}\label{thm:R-approximation-non-compact}
    Let $\P$ and $\Q$ be probability measures on $\mathcal{X}$, such that $\P \ll \Q$ and $R(\P \parallel \Q) < \infty$.
    Then, for any $\varepsilon > 0$, there exists a two hidden layer neural network $T_\theta: \mathcal{X} \to \RR$ such that $R(\P \parallel \Q) - \DV(T_\theta) < \varepsilon$.
\end{theorem}

\begin{proof}
    By the Donsker-Varadhan representation theorem, there exists a continuous, bounded (say by $M' - 1$) function $T: \mathcal{X} \to \RR$ such that $R(\P \parallel \Q) - \DV(T) < \varepsilon / 2$.
    Let it be normalized such that $\E^{\Q}[e^T] = 1$ (note that $\DV(T + c) = \DV(T)$ for $c \in \RR$).
    By \Cref{lem:approximation-non-compact}, with the probability measure $(\P + \Q) / 2$, there exists a two hidden layer neural network $T_\theta$, bounded by $M'$, such that $\E^{\P}[|T - T_\theta|] < \varepsilon / 4$ and $\E^{\Q}[|T - T_\theta|] < e^{-M'} \varepsilon / 4$.
    Then we have that
    
    \begin{equation}\label{eq:R-approximation-non-compact-T-T-theta}
        \begin{split}
            \DV(T) - \DV(T_\theta) &= \E^{\P}[T - T_\theta] + (\log \E^{\Q}[e^{T_\theta}] - \underbrace{\log \E^{\Q}[e^{T}]}_{0}) \\
            &\leq \E^{\P}[T - T_\theta] + (\E^{\Q}[e^{T_\theta}] - \underbrace{\E^{\Q}[e^{T}]}_{1}) \\
            &\leq \E^{\P}[|T - T_\theta|] + \E^{\Q}[|e^{T_\theta} - e^{T}|] \\
            &\leq \E^{\P}[|T - T_\theta|] + \E^{\Q}[e^{M'}|T_\theta - T|] \\
            &< \frac{\varepsilon}{4} + e^{M'} e^{-M'} \frac{\varepsilon}{4} \\
            &= \frac{\varepsilon}{2},
        \end{split}
    \end{equation}

    \noindent where the first inequality follows from the identity $\log x \leq x - 1$, and the third from the Lipschitz constant $e^{M'}$ of the exponential function restricted to $[-M', M']$.
    Combining the two inequalities, we have 

    \begin{equation}
        R(\P || \Q) - \DV(T_\theta) = \left(R(\P || \Q) - \DV(T_\theta)\right) + \left(\DV(T) - \DV(T_\theta)\right) < \frac{\varepsilon}{2} + \frac{\varepsilon}{2} = \varepsilon.
    \end{equation}
\end{proof}

\subsection*{The result}

We are now able to state and prove our main theorem.

\begin{theorem}\label{thm:remedi-consistency-non-compact}
    REMEDI is strongly consistent, up to an arbitrarily small precision $\varepsilon > 0$.
\end{theorem}

\begin{proof}
    Choose the number of parameters $N$, (appropriately assigned to the first and second layer) and a compact set $\Theta \subset \RR^N$, both large enough such that $|R(\P \parallel \Q) - \DV(\F_\Theta)| < \varepsilon / 2$, by \Cref{thm:R-approximation-non-compact}.
    By the triangle inequality,

    \begin{equation}
    \begin{split}
        |R(\P \parallel \Q) - \DVnm(\F_\Theta)| &= |R(\P \parallel \Q) - \DV(\F_\Theta) + \DV(\F_\Theta) - \DVnm(\F_\Theta)| \\
        &\leq |R(\P \parallel \Q) - \DV(\F_\Theta)| + |\DVnm(\F_\Theta) - \DV(\F_\Theta)| \\
        &< \frac{\varepsilon}{2} + |\DVnm(\F_\Theta) - \DV(\F_\Theta)|.
    \end{split}
    \end{equation}

    \noindent The second term converges to zero almost surely by \Cref{thm:theta-uniform-convergence-non-compact}, and therefore the error converges to strictly less than $\varepsilon$ almost surely.
\end{proof}

\section{Justification of the loss function}

In this section, we present the proof of Proposition~\ref{Thm:DV_loss} which provides insights into the \texttt{REMEDI} loss function and it's connection to density estimation. We restate the following notations for their use in the proof. Let $\P$ and $\Q$ be the target and base distribution, respectively, defined on the sample space $\mathcal{X}$, with densities $p, q$, let $X$ be a random variable with distribution $\P$, and $T$ be a function from the class of continuous bounded functions, or the class of Borel-measurable functions from $\mathcal{X}$ to $\RR$. Using the function $T$ we define the Gibbs distribution $\mathbb{G}$ on the sample space of $\mathcal{X}$, which has the density $\tilde{p}(x) = \frac{q(x)e^{T(x)}}{\E^{\Q} [e^T]}$. 
\begin{proposition}\label{Thm:DV_loss}
Given a base distribution $\Q$, assuming $\E^{\Q} [e^T]$ exists, consider the following density defined by $T$,
\begin{align*}
\tilde{p}(x) = \frac{q(x)e^{T(x)}}{\E^{\Q} [e^T]}.
\end{align*}
Then, 
\begin{itemize}
    \item[(i)] the right-hand side in Eq.~\eqref{eq:entropy-est-DV-2} is equal to $-\E^{\P} \log \tilde{p}$
    \item[(ii)] $T^*$ is the solution of 
\begin{align*}
\sup_{T: \mathcal{X} \to \RR} \left(\E^{\P}[T] - \log\E^{\Q}[e^T]\right),
\end{align*}
if and only if the associated density $\tilde{p}^*$ satisfies $\tilde{p}^* = p$ for any base distribution $\Q$.
\end{itemize}
\end{proposition}
\begin{proof}[Proof of Proposition 3.1] We prove Proposition 3.1 in the following two parts.
\vspace{0.5cm}

\noindent 1. \underline{\textit{Showing} $\mathcal{L}_{\texttt{REMEDI}} = -\E^{\P} \log \tilde{p}$ :}
\noindent Using the form of $\tilde{p}$ we can write,
\begin{flalign}
    -\E^{\P} \log \tilde{p}(X) &= - \E^{\P} \log \frac{q(X)e^{T(X)}}{\E^{\Q} [e^T(X)]} \nonumber \\
    &= - \E^{\P} \log q(X) - \left(\E^{\P} T(X) - \log \E_{\mathbb{Q}} e^{T(X)}\right) \nonumber \\
    &= \mathcal{L}_{\texttt{REMEDI}} \nonumber
\end{flalign}
\noindent 2. \underline{\textit{Showing that optimal} $T^*$ \textit{is achieved if and only if} $\tilde{p}^* = p$ \textit{for any} $\Q$ :}
\noindent For any function $T: \mathcal{X} \to \RR$, and corresponding, previously defined, distribution $\mathbb{G}$, we expand the following relative entropy,
\begin{flalign}\label{eq:expandrelentr}
    R(\P || \mathbb{G}) &= \E^{\P} \log \frac{p(X)}{\tilde{p}(X)} \nonumber \\
    &= \E^{\P} \log \frac{p(X)}{\frac{q(X)e^{T(X)}}{\E^{\Q} [e^T(X)]}} \nonumber \\
    &= \E^{\P} \log \frac{p(X)}{q(X)} - \E^{\P} T(X) + \log \E_{\mathbb{Q}} e^{T(X)}
\end{flalign}
\noindent \textbf{If part :}
We show that, for any $T^*$, if the associated Gibbs density $\tilde{p}^*$ is equal to $p$ then $T^*$ is the solution of $ \sup_{T: \mathcal{X} \to \RR} \left(\E^{\P}[T] - \log\E^{\Q}[e^T]\right)$. 

Let, $\mathbb{G}^*$ be the distribution with density $\tilde{p}^*$. Then we have, 
\begin{flalign}
    \tilde{p}^* = p \implies R(\P || \mathbb{G}^*) = 0 \implies \E^{\P} \log \frac{p(X)}{q(X)} = \E^{\P} T^*(X) - \log \E_{\mathbb{Q}} e^{T^*(X)} \nonumber
\end{flalign}
The last equality is due to Eq.~\eqref{eq:expandrelentr}. Since the LHS of the last equation is $R(\P || \Q)$ following Donsker-Varadhan representation \cite{donsker1983asymptotic} we have,
\begin{flalign}
    R(\P || \Q) = \E^{\P} T^*(X) - \log \E_{\mathbb{Q}} e^{T^*(X)} \leq \sup_{T: \mathcal{X} \to \RR} \left(\E^{\P}[T] - \log\E^{\Q}[e^T]\right) =  R(\P || \Q) \nonumber
\end{flalign}
Therefore, $T^*$ is the maximizer of $\left(\E^{\P}[T] - \log\E^{\Q}[e^T]\right)$.

\noindent \textbf{Only If part :} We show that, if $T^*$ is the solution of $\sup_{T: \mathcal{X} \to \RR} \left(\E^{\P}[T] - \log\E^{\Q}[e^T]\right)$ then $\tilde{p}^* = p$.

Let, $\mathbb{G}^*$ be the distribution with density $\tilde{p}^*$. From the Donsker-Varadhan representation, we have,
\begin{flalign}\label{eq:DVrepr}
    R(\P || \Q) = \sup_{T: \mathcal{X} \to \RR} \left(\E^{\P}[T] - \log\E^{\Q}[e^T]\right) = \E^{\P} T^*(X) - \log \E_{\mathbb{Q}} e^{T^*(X)}
\end{flalign}
Therefore, we have,
\begin{flalign}
    R(\P || \mathbb{G}^*) = \E^{\P} \log \frac{p(X)}{q(X)} - \E^{\P} T^*(X) + \log \E_{\mathbb{Q}} e^{T^*(X)} = 0 \nonumber
\end{flalign}
The first equality is due to Eq.~\eqref{eq:expandrelentr} and the second is from Eq.~\eqref{eq:DVrepr}. Hence the proof follows.

\end{proof}

\section{Additional synthetic benchmarks}\label{sec:additional-synthetic-benchmarks}

To benchmark how REMEDI copes with increasing dimensionality, we provide two additional benchmarks, the $d$-dimensional ball and hypercube, see Section \ref{sec:datasets}.
Since the investigation concerns the high-dimensional performance on the REMEDI correction and not Gaussian mixture models (see Section \ref{sec:inefficiency-of-knife} for that analysis), we stick to a single component KNIFE base model, which also trains quickly
For comparison, we juxtapose the results with the validation set cross-entropy estimate of the $256$-component KNIFE model, which performs the best on almost all benchmarks in Section \ref{sec:inefficiency-of-knife}.

From Figures \ref{fig:loss_ball} and \ref{fig:loss_cube}, we note that the \texttt{REMEDI} entropy estimates are consistently considerably closer to the true entropy than what the KNIFE framework can achieve.
Further, although some overfitting can be seen, it is nowhere near as bad as for KNIFE (see \Cref{sec:inefficiency-of-knife}). This means that there is likely more for \texttt{REMEDI} to claim with additional hyperparameter tuning.

\begin{figure}[H]
	\centering
	\begin{subfigure}{0.48\textwidth}
    	\centering
    	\includegraphics[width=\linewidth]{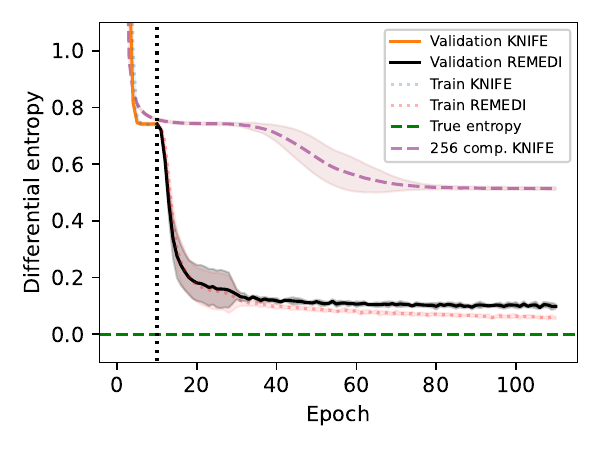}
        \label{fig:loss_ball8d}
	\end{subfigure}
	\begin{subfigure}{0.48\textwidth}
    	\centering
    	\includegraphics[width=\linewidth]{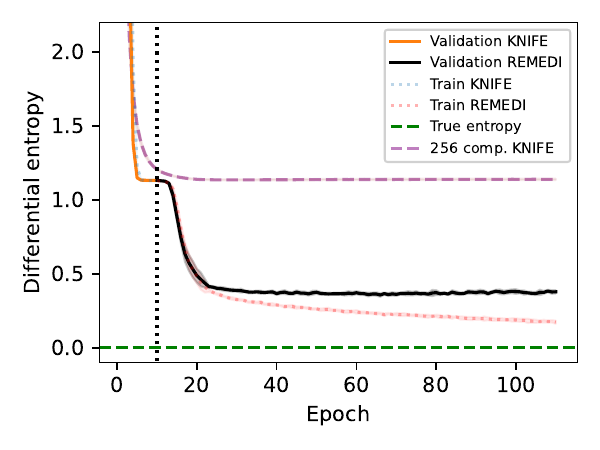}
        \label{fig:loss_ball20d}
    \end{subfigure}
    \caption{Training curves on ball dataset. The $8$-dimensional dataset is shown to the left and the $20$-dimensional to the right. The 256-component KNIFE model is also shown for reference.}
    \label{fig:loss_ball}
\end{figure}

\begin{figure}[H]
	\centering
	\begin{subfigure}{0.48\textwidth}
    	\centering
    	\includegraphics[width=\linewidth]{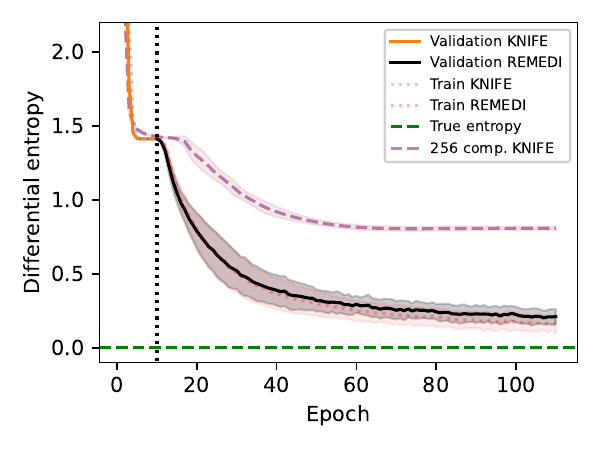}
        \label{fig:loss_cube8d}
	\end{subfigure}
	\begin{subfigure}{0.48\textwidth}
    	\centering
    	\includegraphics[width=\linewidth]{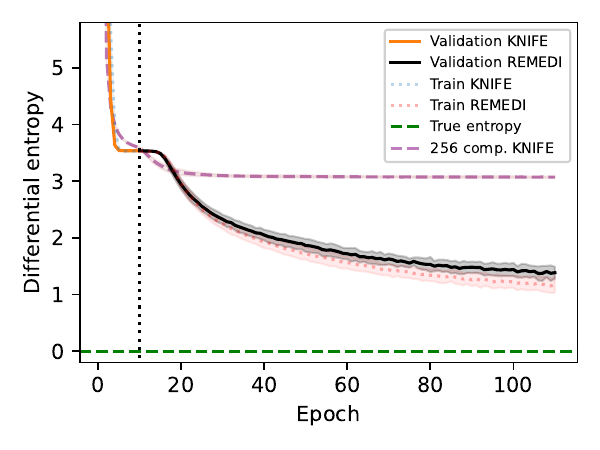}
        \label{fig:loss_cube20d}
    \end{subfigure}
    \caption{Training curves on cube dataset. The $8$-dimensional dataset is shown to the left and the $20$-dimensional to the right. The 256-component KNIFE model is also shown for reference.}
    \label{fig:loss_cube}
\end{figure}

\subsection{Inefficiency of KNIFE}\label{sec:inefficiency-of-knife}

An often-cited folklore result in statistics is that Gaussian mixture models are dense in the space of probability measures equipped with the weak-* topology.
Parzen-windowing is a subset of Gaussian mixture models with the same property.
As a consequence, both model families can be considered for a wide range of tasks such as density and entropy estimation.
\cite{ahmad1976nonparametric} give conditions for when a certain Parzen-windowing-based estimator converges to the differential entropy.
It is however well known that standard kernel density estimation suffers heavily from the curse of dimensionality, see e.g. the table on page 319 of \cite{wasserman2004all}, which shows that an unreasonable amount of data points is required to obtain low error on a multi-variate normal target $\P$.
Since Gaussian mixture models such as KNIFE have more flexibility, they require less components to achieve the estimate.
The question then remains whether they manage to be sufficiently data efficient in growing dimension.
We show empirically that the KNIFE approach suffers from similar sample inefficiency problems, even in a moderate dimension.

Unlike \cite{wasserman2004all} we cannot simply use a multi-variate normal as the target, since this is in the model class of KNIFE for any amount of components.
Instead, we opt for the $8$-dimensional targets comprised of uniforms over a hypercube/ball, as well as the $8$-dimensional triangle dataset, with $50000$ training and validation samples each, see Section \ref{sec:datasets}.

Fig.~\ref{fig:8d_varcomps_losses} and Fig.~\ref{fig:cube_8d_varcomps_losses} show that KNIFE hits a barrier on all datasets when increasing the number of components.
\begin{wrapfigure}{r}{0.5\textwidth}
    \centering
    \includegraphics[width=0.4\textwidth]{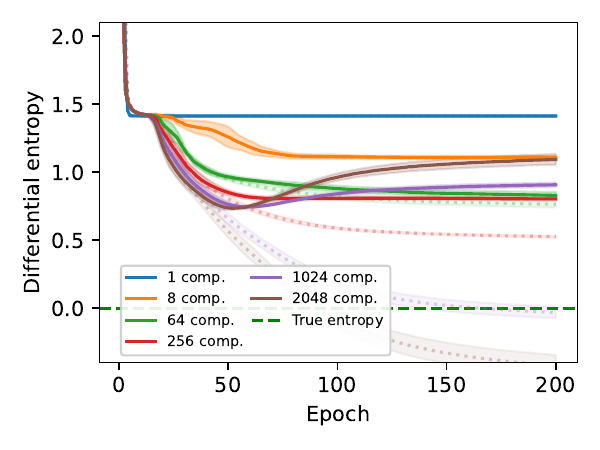}
    \caption{$8$-dimensional cube}
    \label{fig:cube_8d_varcomps_losses}
\end{wrapfigure}
Increasing the number of components further leads to overfitting, implying that KNIFE is data-inefficient already in this dimension, especially on the cube and ball datasets where the best estimates are comparable with using just one component.

To give additional verification to this, we perform the experiment again in $20$ dimensions, shown in 
Fig.~\ref{fig:20d_varcomps_losses}.
We note that additional components at best give an almost negligible improvement to all three datasets; in fact, we see no gain at all on the ball dataset.

It becomes clear that in both of these dimensions, that \texttt{REMEDI} can provide much more valuable entropy estimates, at least outside of gigantic sample sizes.
Still, \texttt{REMEDI} like most algorithms will of course also suffer from problems like overfitting and the curse of dimensionality.
However, the prior provided by the class of shallow neural networks trained with the Donsker-Varadhan target scales better to moderately high dimensions that Parzen-windowing and KNIFE struggle with.

\begin{figure}[H]
	\centering
    \begin{subfigure}{0.33\textwidth}
    	\centering
    	\includegraphics[width=\linewidth]{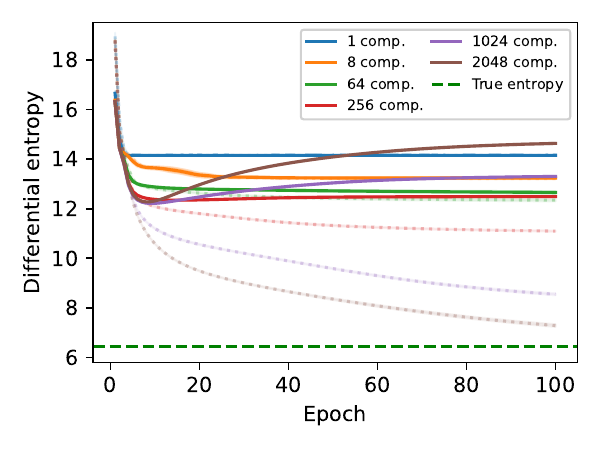}
    	\caption{Triangle}
        \label{fig:triangle_20d_varcomps}
	\end{subfigure}
	\begin{subfigure}{0.33\textwidth}
    	\centering
    	\includegraphics[width=\linewidth]{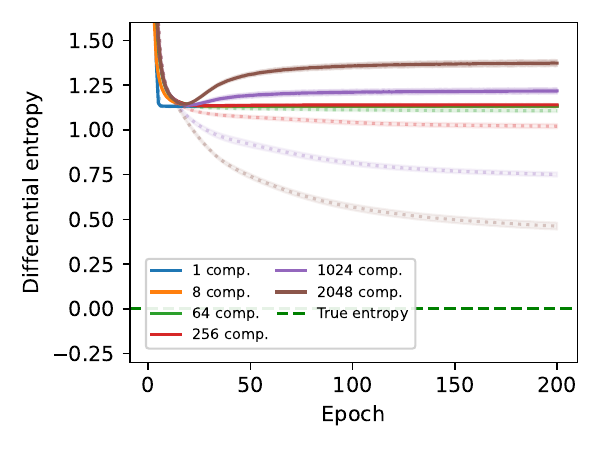}
    	\caption{Ball}
        \label{fig:ball_20d_varcomps_losses}
	\end{subfigure}
	\begin{subfigure}{0.33\textwidth}
    	\centering
    	\includegraphics[width=\linewidth]{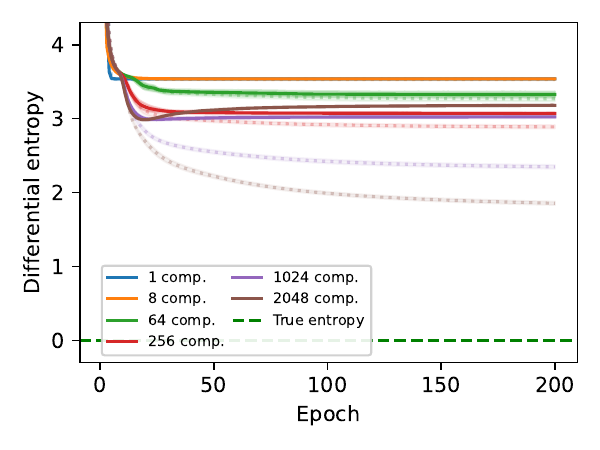}
    	\caption{Cube}
        \label{fig:cube_20d_varcomps_losses}
    \end{subfigure}
    \caption{KNIFE training curves on three $20$-dimensional datasets.}
    \label{fig:20d_varcomps_losses}
\end{figure}

\subsection{Impact of the base distribution}\label{sec:impact-of-the-base-distribution}

One of the primary innovations of \texttt{REMEDI} compared to MINE or DDDE is the adaptive base distribution.
The base distribution is optimized for cross-entropy, or equivalently in the large sample regime, to minimize the relative entropy

\begin{equation}
    R(\P || \Q) = \E^{\P}\left[\log \left(\frac{d\P}{d\Q}\right)\right].
\end{equation}

\noindent This means that, on average, the samples from $\Q$ are closer to the support $\P$, and that regions where $\P$ is stronger than $\Q$, i.e. having large $\frac{d\P}{d\Q}$, are less pronounced.
As \cite{mcallester2020formal} points out, these regions are hard to learn and account for in the Donsker-Varadhan estimate.

\begin{figure}[H]
	\centering
	\begin{subfigure}{0.48\textwidth}
    	\centering
    	\includegraphics[width=\linewidth]{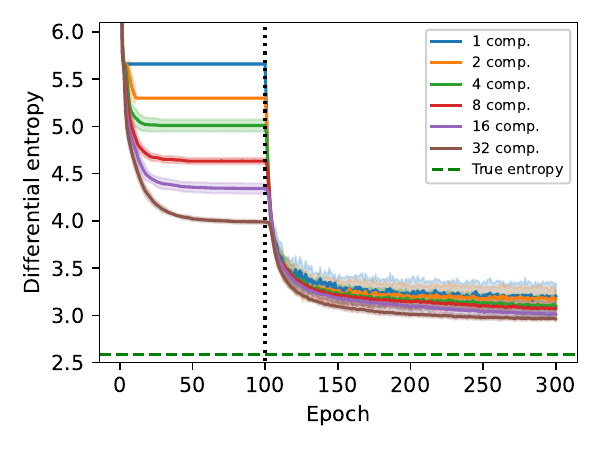}
        \label{fig:triangle_8d_varcomp_unzoomed}
	\end{subfigure}
	\begin{subfigure}{0.48\textwidth}
    	\centering
    	\includegraphics[width=\linewidth]{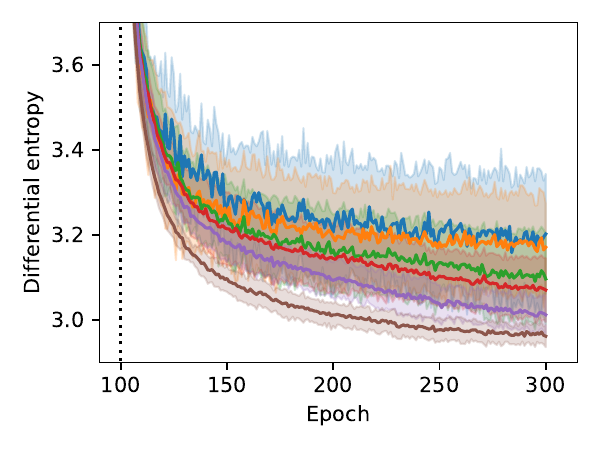}
        \label{fig:triangle_8d_varcomp_zoomed}
    \end{subfigure}
    \caption{Training curves on 8-dimensional triangle dataset for different amounts of components in the base distribution, averaged over 10 runs. The right plot is zoomed into the \texttt{REMEDI} phase.}
    \label{fig:triangle_8d_varcomp}
\end{figure}

\begin{table}[H]
    \centering
    \begin{tabular}{|l|l|l|}
    \hline
        Components & KNIFE & \texttt{REMEDI} \\
        \hline
        1 & $5.6612 \pm 0.0035$ & $3.2025 \pm 0.1414$ \\
        \hline
        2 & $5.2999 \pm 0.0045$ & $3.1713 \pm 0.1039$ \\
        \hline
        4 & $5.0095 \pm 0.0632$ & $3.0968 \pm 0.1068$ \\
        \hline
        8 & $4.6331 \pm 0.0331$ & $3.0708 \pm 0.0750$ \\
        \hline
        16 & $4.3413 \pm 0.0541$ & $3.0124 \pm 0.0449$ \\
        \hline
        32 & $3.9894 \pm 0.0214$ & $2.9621 \pm 0.0026$ \\
        \hline
    \end{tabular}
    \caption{Entropy estimates and standard deviations on 8-dimensional triangle dataset, based on 10 runs.}
    \label{tab:triangle_8d_varcomp}
\end{table}

To investigate what impact a good base distribution has, we rerun the experiments on the 8-dimensional triangle dataset.
In order to get a better understanding of the long-term learning behavior, we double the number of epochs for both the base distribution and the \texttt{REMEDI} to $100$ and $200$, respectively.
In Fig.~\ref{fig:triangle_8d_varcomp} and Table~\ref{tab:triangle_8d_varcomp} we see that learning with a higher amount of components is indeed easier.
The estimates given by using 16 and 32 components clearly beat the others, and seem to be learning at a faster pace.
Surprisingly, the difference is not as pronounced between 1 and 2 components, but we do note that the training stabilizes for more components, also there.

\subsection{Comparison to normalizing flows}

Other than mixture models such as KNIFE, more flexible methods of density estimation, such as normalizing flows, can be used for entropy computation.
This allows for better estimates, at the cost of computational efficiency.
Therefore, we provide a comparison between the performance of KNIFE/\texttt{REMEDI} and RealNVP \cite{dinh2016density}, a popular family of normalizing flows, at different computational budgets, measured in wall-clock time.
For RealNVP, we use an out-of-the-box implementation provided by the 'normflows' \cite{stimper2023normflows} package.
In Fig.~\ref{fig:pareto-triangle_8d}, the methods are compared on the multi-modal 8-dimension triangle dataset, for different depths (i.e. amount of flow steps) of RealNVP, while the \texttt{REMEDI} settings are untouched from earlier experiments.
Taking the lower convex hull of all wall time-entropy estimate pairs, we see that KNIFE/\texttt{REMEDI} occupies practically the entire Pareto front.
Note that there are other families of normalizing flows, continuous normalizing flows (CNF) \cite{grathwohl2018ffjord} such as diffusion models \cite{song2020score, song2020denoising, kingma2021variational} that may be able to better cope with this multi-modality, but since these require integrating the divergence of a vector field along the ODE solution \cite{chen2018neural}, they have high complexity and we consider them out of scope.

\begin{figure}[H]
    \centering
    \includegraphics[width=0.5\textwidth]{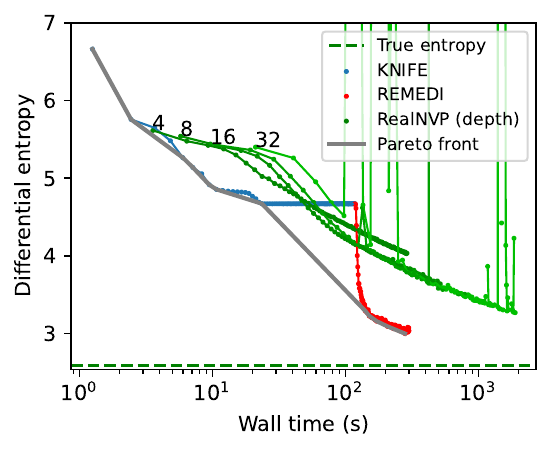}
    \caption{Performance of REMEDI and KNIFE given computational budget, compared to RealNVP.}
    \label{fig:pareto-triangle_8d}
\end{figure}

\section{Experimental details}\label{sec:experimental-details}

\subsection{Entropy estimation experiments}

 For the experiments on the triangular and two moons datasets, we opted for a neural network that outputs $f_\theta(x) = e^{T(x)}$ in the Donsker-Varadhan formula, letting $T(x)$ be obtained by taking the logarithm.
Following \cite{park2021deep}, this is implemented using the $\mathrm{ELU}$ activation function \cite{clevert2015fast} with $\alpha = 1$, by applying the transformation

\begin{equation}\label{eq:ELU-final-transformation}
    \mathrm{ELU}(x) + 1 + \epsilon
\end{equation}

\noindent to the final output of the model, where $\epsilon$ is a small scalar.

The architecture here exploits the structure of the KNIFE base distribution.
For each of the $M$ components $i$, with weight $\alpha_i$, mean $\mu_i$ and covariance $\Sigma_i$ (and precision $\Lambda_i = \Sigma_i^{-1}$), the input $x$ is projected onto the uncorrelated components, with respect to $\Sigma_i$, by premultiplying by the lower triangular Cholesky factor $L_i$, of $\Lambda_i$.
Hence, we obtain $M$ component-wise decorrelated offset vectors

\begin{equation*}
    y_i = L_i (x - \mu_i) \in \RR^d.
\end{equation*}

\noindent These are then transformed via learnable matrices $A_i \in \RR^{d \times d}$, since the features produced by $L_i$ are not consistent between dimensions, even when performing singular value decomposition.
These are then propagated through a fully connected intermediate network with two (three for the 20-dimensional hypercube) shared linear layers 
followed by ReLU activation functions, reaching the penultimate layer with dimension $d_p$ (in most experiments $500$).
They are then scalar multiplied with vectors $b_i \in \RR^{d_p}$ to produce scalars, which are then combined via weighting by the component relevances

\begin{equation}
    p(i | x) = \frac{p(i)p(x|i)}{p(x)} \propto w_i \exp\left(-\frac{1}{2} ||L_i(x - \mu_i)||^2\right).
\end{equation}

\noindent This weighted sum is finally input to the transformation in Eq.~\eqref{eq:ELU-final-transformation}.
In Table~\ref{tab:hyperparams-entropy-estimation}, the specific setups for each of the entropy estimation tasks are shown.

\begin{table}[h!]
\center
\resizebox{0.5\columnwidth}{!}{
\begin{tabular}{|c|c|cc|cc|c|c|}
\hline
\textbf{Hyperparameter} & \textbf{Two moons} & \multicolumn{2}{c|}{\textbf{Triangle}} & \multicolumn{2}{c|}{\textbf{Ball}} & \multicolumn{2}{c|}{\textbf{Hypercube}}  \\
\hline
Dimension & 2 & 1 & 8 & 8 & 20 & 8 & 20 \\
Train set size & 50,000 & 50,000 & 50,000 & 50,000 & 50,000 & 50,000 & 50,000 \\

Validation set size & 50,000 & 50,000 & 50,000 & 50,000 & 50,000 & 50,000 & 50,000 \\

\# KNIFE components & 8 & 16 & 16 & 1 & 1 & 1 & 1 \\

Intermediate network layer widths & (500, 500) & (500, 500) & (500, 500) & (200, 200) & (200, 200) & (1000, 1000, 500) & (1000, 1000, 500) \\

\# epochs KNIFE & 50 & 50 & 50 & 10 & 10 & 10 & 10 \\

\# epochs \texttt{REMEDI} & 100 & 100 & 100 & 100 & 100 & 100 & 100 \\

Training batch size & 1000 & 1000 & 1000 & 1000 & 1000 & 1000 & 1000 \\

Optimizer & Adam & Adam & Adam & Adam & Adam & Adam & Adam \\

Learning rate & 1e-3 & 1e-3 & 1e-3 & 1e-3 & 1e-3 & 1e-4 & 1e-4 \\

Weight decay & 1e-4 & 1e-4 & 1e-4 & 1e-4 & 1e-4 & 1e-4 & 1e-4 \\
\hline
\end{tabular}
}
\vspace{0.2cm}
\caption{Hyperparameter settings used for the entropy estimation experiments.}
\label{tab:hyperparams-entropy-estimation}
\end{table}

\subsection{Information Bottleneck experiments}
\label{sec:additionaldetailsexperiments}
For all the Information Bottleneck (IB) experiments, we have used the encoder-decoder architecture from \cite{pmlr-v206-samaddar23a}. On MNIST, we used an MLP encoder with three fully connected layers. The first two layers each contain 800 nodes with ReLU activations and the last layer has $2K$ nodes predicting the $\mu(X)$ and diagonal of $\Sigma(X)$ of the encoder distribution. On CIFAR-10, we used a VGG16 encoder. For both datasets, we used a single-layer neural network as the decoder. We chose the latent space dimension $K=32$ for both datasets. Note that, we apply the square transformation of $\operatorname{MI}(X;Z)$ term in the IB objective for all methods. This transformation makes the solution to IB objective function identifiable with respect to $\beta$ \cite{Galvez2020}. For evaluation, we chose the model at the final epoch on MNIST and the model with the best validation loss on CIFAR-10 and ImageNet. 

For the CIFAR-10 data following \cite{pmlr-v206-samaddar23a}, we perform a data augmentation step before training where we augmented the training data using random transformations. We use the padding of each training data point by 4 pixels on all sides and crop at a random location to return an original-sized image. We  perform a flip of each training set image horizontally with a probability of 0.5. Furthermore, we transform the training and validation set image with mean = $(0.4914, 0.4822, 0.4465)$ and standard deviation = $(0.2023, 0.1994, 0.2010)$.

For ImageNet, we resize the input images to 299 $\times$ 299 pixels by cropping them at their center. Subsequently, we normalize the images to achieve a mean of (0.5, 0.5, 0.5) and a standard deviation of (0.5, 0.5, 0.5). Our approach aligns with the implementation of \cite{alemi2016deep}, where we apply a transformation to the ImageNet data using a pre-trained Inception Resnet V2 \cite{szegedy16} network, excluding the output layer. This transformation results in the original ImageNet images being reduced to a 1534-dimensional representation, which serves as the basis for all our obtained results. In accordance with \cite{alemi2016deep}, we employ an encoder featuring two fully connected layers, each comprising 1024 hidden units, along with a single-layer decoder architecture. We chose the latent space dimension $K=100$ for ImageNet.

The implementations of MINE and \texttt{REMEDI} require us to specify a neural network to approximate the function $T$. For MINE, following the implementation of \cite{belghazi2018mutual} we chose a two-layer MLP with 512 nodes each and additive Gaussian noise and ELU activations. For \texttt{REMEDI}, we chose a two-layer MLP network with 100 nodes per layer and ReLU activations. 

Implementing the mutual information estimators of KNIFE, MINE, and \texttt{REMEDI} requires us to perform sub-optimization of their parameters inside the main optimization of the encoder-decoder parameters. We follow the IB implementations of \cite{pichler2022differential} to freeze the parameters of the encoder-decoder before performing the sub-optimization of the mutual information estimators. For KNIFE and MINE, we run the sub-optimization for 5 epochs. We train the mutual information estimators and the encoder-decoder parameters using the same mini-batch.

For \texttt{REMEDI} with KNIFE base distribution, following Algorithm 1, we first train the KNIFE parameters for 5 epochs and then fix the KNIFE parameter to train the parameters of the \texttt{REMEDI} network for 5 epochs. Note that, although ImageNet has 1000 classes, we keep the choice of 10-component KNIFE consistent due to the heavy computational burden of fitting a KNIFE with many components in high dimensions. We use a constant learning rate of 0.001 for training the mutual information estimators.  

For \texttt{REMEDI} with standard Gaussian base distribution, since we are facing a more challenging learning problem we train the parameters of the \texttt{REMEDI} network for 30 epochs with a reduced learning rate of 0.0001. In addition, during training, we consider an initial burn-in period of 5 epochs where we don't train the \texttt{REMEDI} network and introduce it after epoch 5. In our experiments, this improves the stability of the algorithm. Additional hyperparameter details regarding the encoder-decoder training are described in the below Table~\ref{tab:hyperparams-IB}. 

\begin{table}[h!]
\center
\resizebox{0.5\columnwidth}{!}{
\begin{tabular}{|c|c|c|c|}
\hline
\textbf{Hyperparameters} & \textbf{MNIST} & \textbf{CIFAR-10} & \textbf{ImageNet} \\
\hline
Train set size & 60,000 & 50,000 & 128,1167 \\

Validation set size & 10,000 & 10,000 & 50,000 \\

\# epochs & 100 & 400 & 200 \\

Training batch size & 200 & 200 & 2000  \\

Optimizer & Adam & SGD & Adam  \\

Learning rate & 1e-4 & 0.1 & 1e-4 \\

Learning rate drop & 0.6 & 0.1 & 0.97 \\

Learning rate drop steps & 10 epochs & 100 epochs & 2 epochs \\

Weight decay & Not used & 5e-4 & Not used\\
\hline
\end{tabular}
}
\vspace{0.2cm}
\caption{Hyperparameter settings used for the IB experiments.}
\label{tab:hyperparams-IB}
\end{table}

\subsection{Datasets}\label{sec:datasets}

\subsubsection{Triangular}

The triangular dataset is structurally the same as the one appearing in \cite{pichler2022differential}. 
It is defined for any dimension $d > 1$ as the $d$-fold product distribution of bimodal the distribution pictured in Fig.~\ref{fig:triangle-marginal-supplemental} with itself, making it a $2^d$-modal distribution.
In one dimension, we use the 10-modal distribution in Fig.~\ref{fig:triangle-1d-supplemental}, to match \cite{pichler2022differential}.

\begin{figure}[H]
	\centering
	\begin{subfigure}{0.48\textwidth}
    	\centering
    	\includegraphics[width=\linewidth]{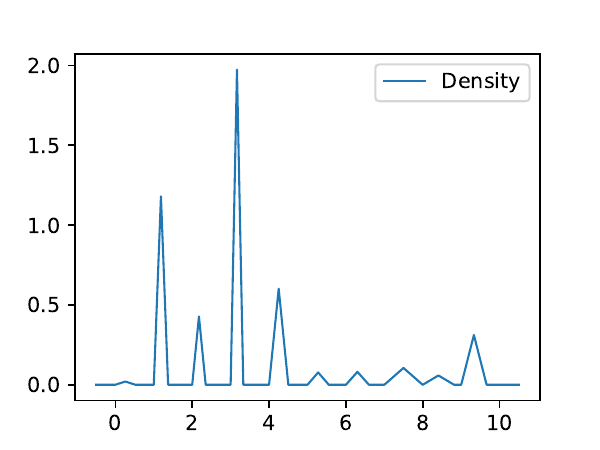}
    	\caption{One-dimensional data distribution.}
        \label{fig:triangle-1d-supplemental}
	\end{subfigure}
	\begin{subfigure}{0.48\textwidth}
    	\centering
    	\includegraphics[width=\linewidth]{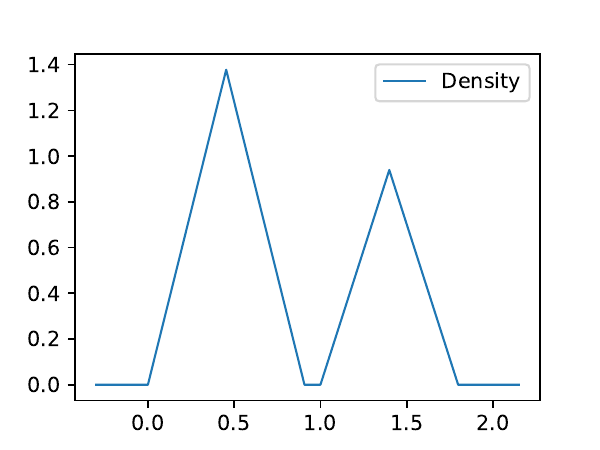}
    	\caption{Marginal data distribution, for $d > 1$.}
        \label{fig:triangle-marginal-supplemental}
    \end{subfigure}
    \caption{Triangle dataset.}
    \label{fig:triangle-supplemental}
\end{figure}

\subsubsection{Two moons}\label{sec:two-moons-dataset}

The two moons dataset consists of samples from $\mathtt{sklearn.datasets.make\_moons()}$, from Scikit-learn \cite{scikit-learn}, with a noise level of $0.05$.
$5000$ samples are plotted in Fig.~\ref{fig:two-moons-samples-supplemental}.
The entropy of this dataset does not offer a closed-form expression.
To have an oracle baseline, we thus take a million samples from the dataset and run a kernel density estimator, with bandwidth $0.01$, tested against one hundred thousand independent samples.
This yields an entropy estimate of $0.2893$, with a standard error of $0.0022$.
As an additional check, we run a Kozachenko-Leonenko $k$-nearest neighbor estimator \cite{kozachenko1987sample, gao2018demystifying} on one hundred million samples, setting $k = 10$ and using the Euclidean distance, which yields a value of $0.2892$.
Given the consistency properties of these estimators, we can confidently say that the true entropy is close to $0.29$.

\begin{figure}[H]
    \centering
    \includegraphics[width=0.5\textwidth]{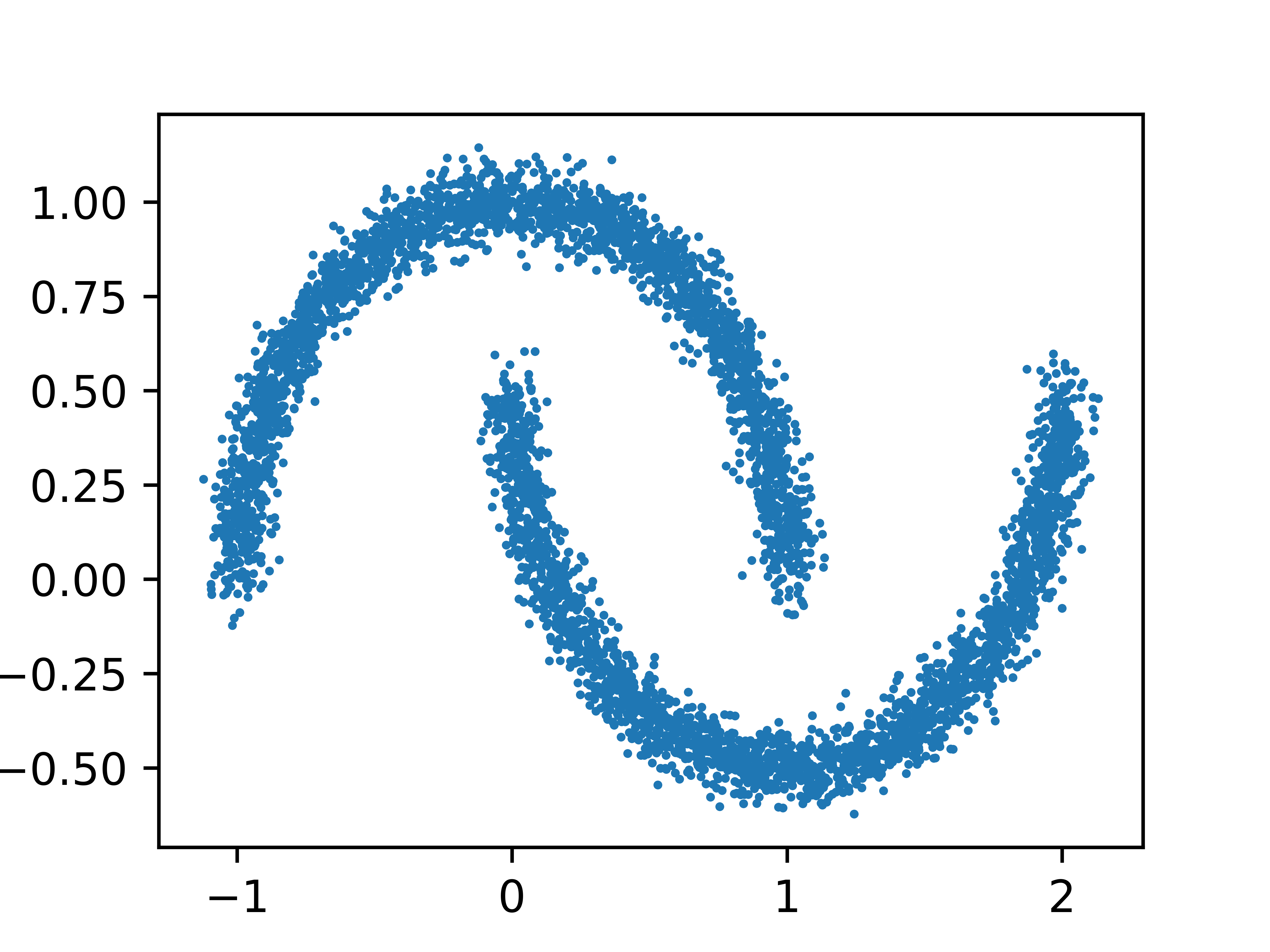}
    \caption{Samples from the two moons dataset.}
    \label{fig:two-moons-samples-supplemental}
\end{figure}

\subsubsection{Uniform hypercube/ball}

The datasets used to compare dimensional scaling between KNIFE and REMEDI are from the uniform distribution over a centered unit volume $d$-dimensional ball $B^d$ and cube $H^d$, respectively.
This is easily scaled to any dimension $d$.
By the unit volume, the true differential entropy of each dataset is $0$, for all dimensions.

\section{Code}

The code needed to replicate the experiments of this paper is found at \url{https://github.com/viktor765/REMEDI}. The project also uses code from the KNIFE repository, with the authors' permission, found at \url{https://github.com/g-pichler/knife/}.

\section{Computational cost}
The suite of pure entropy estimation tasks, including the two moons, triangle, ball, and hypercube datasets were run on NVIDIA A100 and finished in approximately 29 hours.
All IB experiments were run also using NVIDIA A100 GPUs. One replication of IB with \texttt{REMEDI} run for a single $\beta$ value took close to 2.5 hours for MNIST, 15 hours for CIFAR-10, and 35 hours for ImageNet.  

\section{Additional results on Information Bottleneck}
In this section, we provide additional results from applying different mutual information estimation approaches to the Information Bottleneck.

\subsection{Results based on log-likelihood}
\label{sec:results_loglike}

In this section, we evaluate the IB methods based on the log-likelihood metric on the three datasets. Fig~\ref{fig:loglike}, shows the plot of the log-likelihood vs the Lagrange multiplier $\beta$ for the IB methods. We observe that similar to classification accuracy results in the main paper \texttt{REMEDI} exhibits the highest log-likelihood on MNIST and ImageNet for most $\beta$ values especially those around the region where the log-likelihood starts to decrease. On CIFAR-10, although all the methods perform similarly in terms of log-likelihood \texttt{REMEDI} with Gaussian base distribution produces the highest log-likelihood.  

\begin{figure*}
    \centering
	\centering
	\begin{subfigure}{0.28\textwidth}
	\centering
	\includegraphics[width=\linewidth]{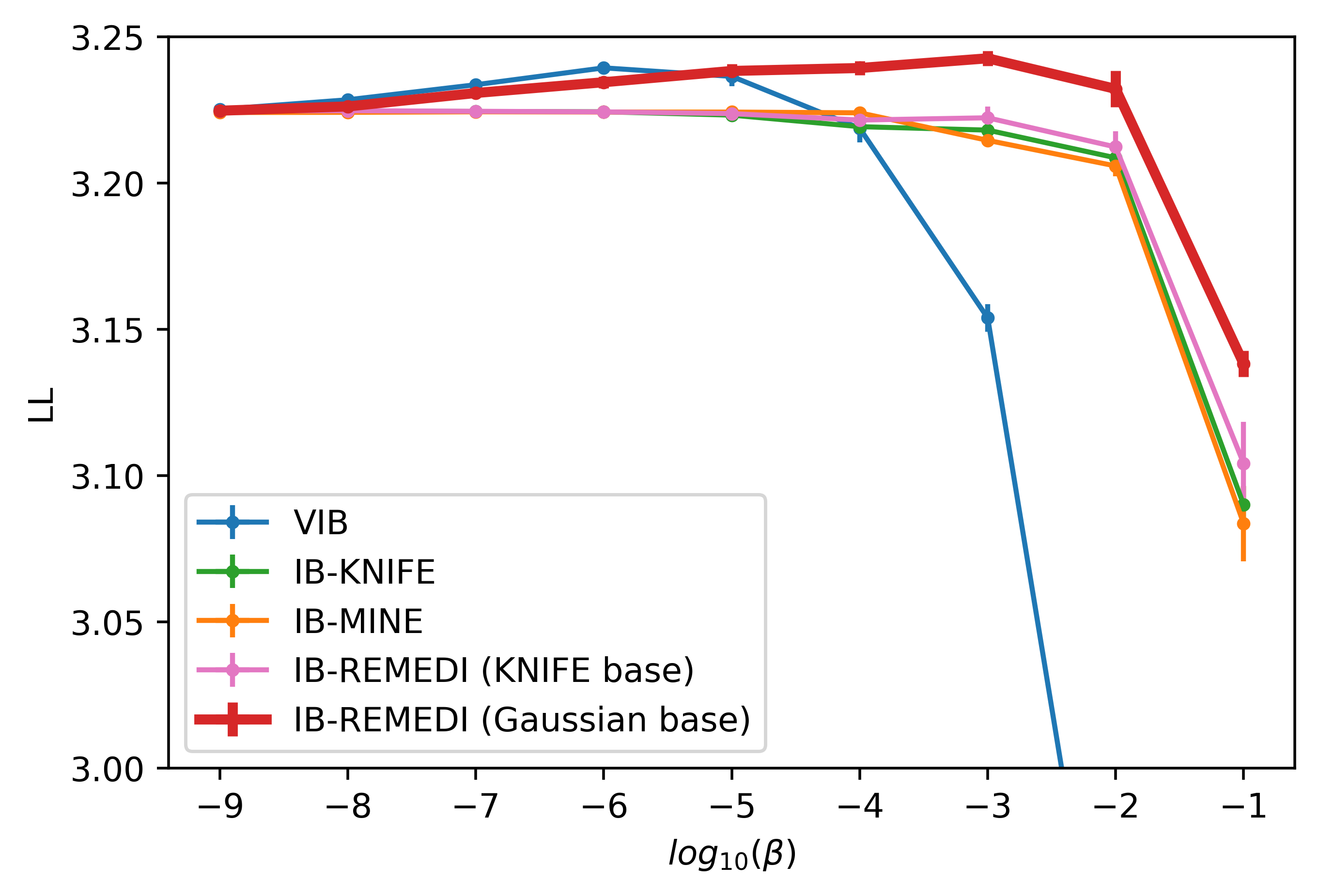}
	\caption{MNIST}
	\end{subfigure}
        \hspace{1cm}
	\begin{subfigure}{0.28\textwidth}
	\centering
	\includegraphics[width=\linewidth]{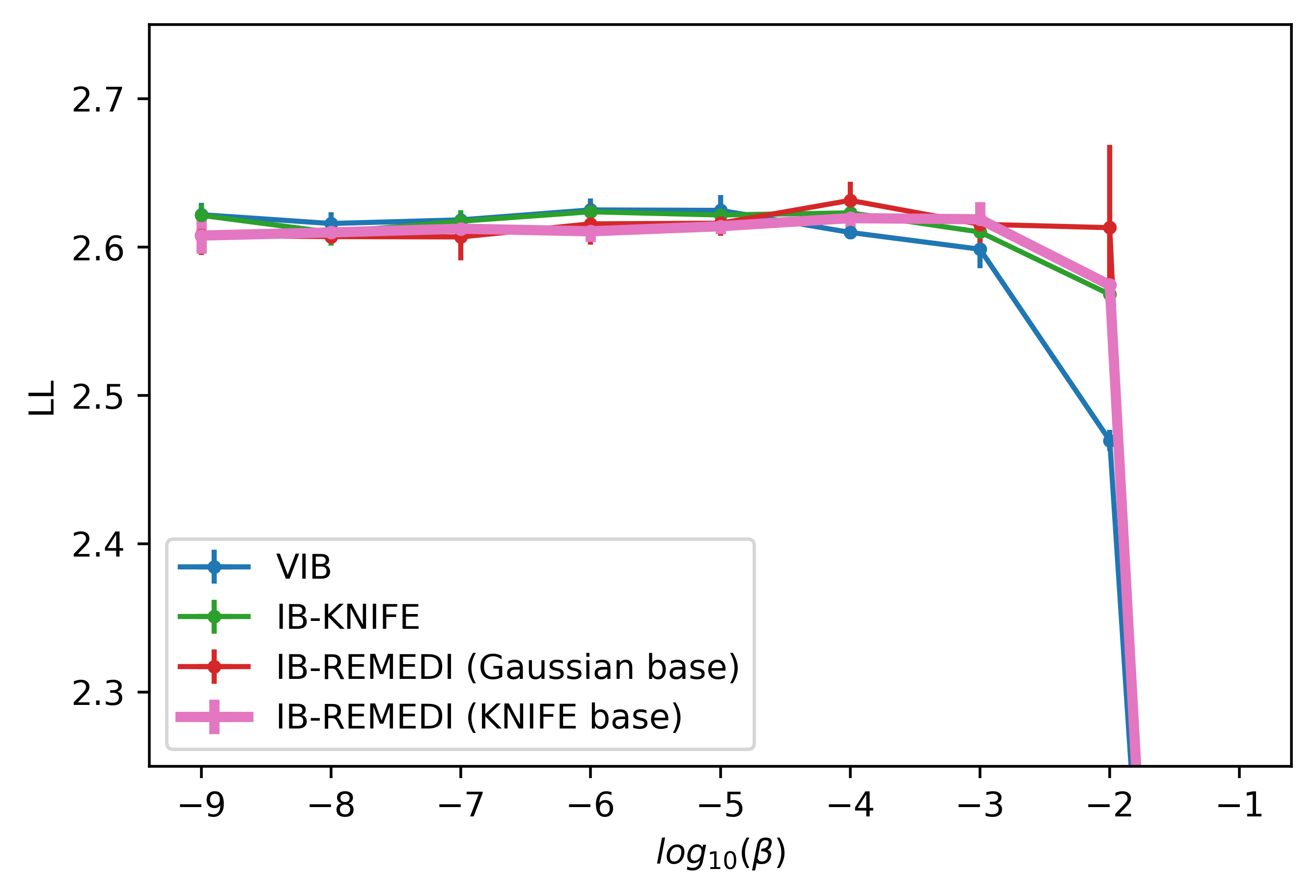}
	\caption{CIFAR10}
	\end{subfigure}
        \hspace{1cm}
	\begin{subfigure}{0.28\textwidth}
	\centering
	\includegraphics[width=\linewidth]{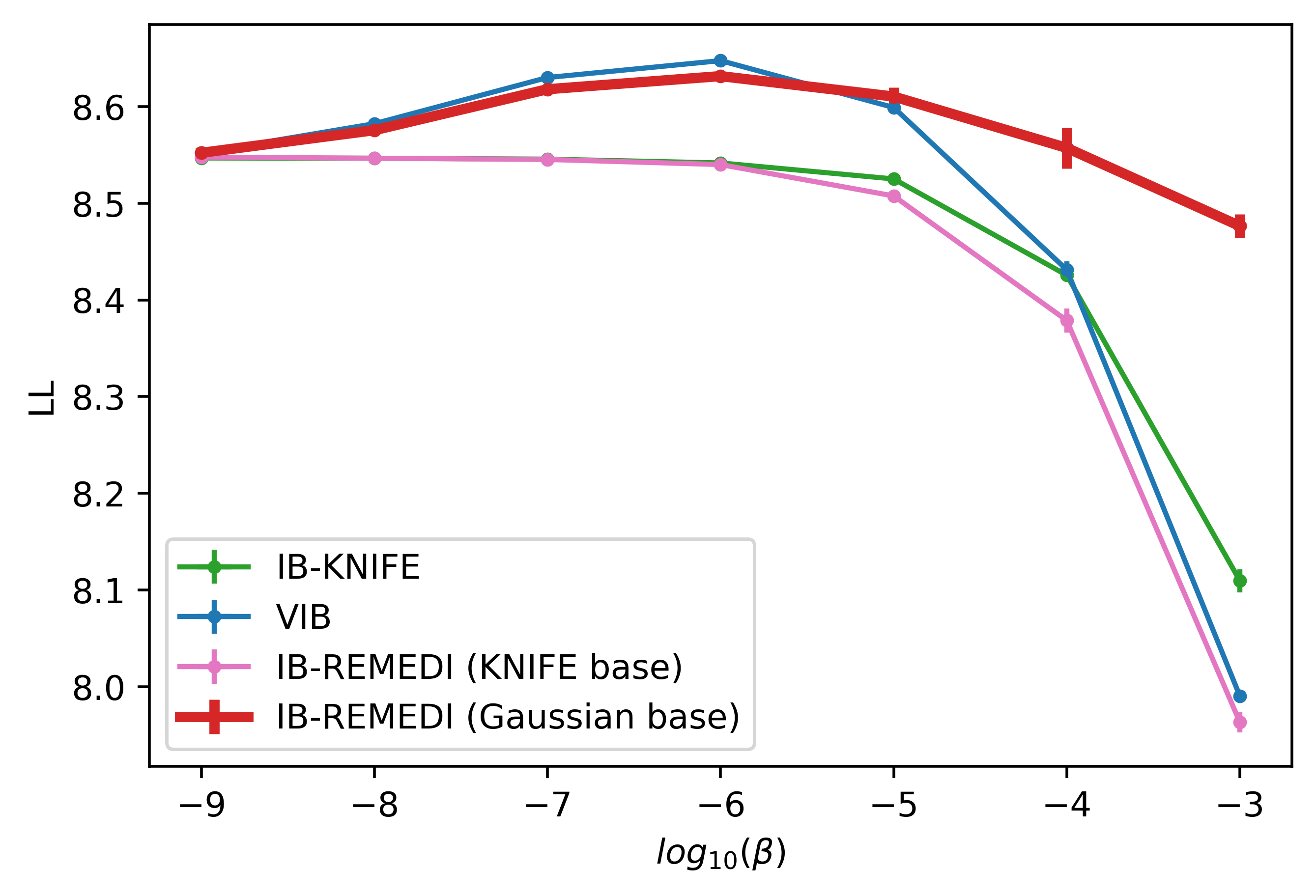}
	\caption{ImageNet}
	\end{subfigure}
    \caption{Plot showing log-likelihoods of the Information Bottleneck methods vs $\beta$ on benchmark image classification datasets (error bars represent standard deviations). For most $\beta$ values, consistently \texttt{REMDI} performs better than other methods on MNIST and ImageNet. On CIFAR10, the classification errors are similar for all the methods. However, \texttt{REMEDI} exhibits the highest log-likelihood across the $\beta$ values. }
    \label{fig:loglike}
\end{figure*}

\subsection{Latent space evolution on MNIST}
\label{sec:latentspace_evolution}
\begin{table}
  \centering
  \begin{tabular}{|M{1cm}|M{4cm}|M{4cm}|M{4cm}|}
    \hline
    \textbf{Epoch} & \textbf{Latent space} & \textbf{KNIFE} & \textbf{\texttt{REMEDI}} \\
    \hline
    0 & \includegraphics[width=0.25\textwidth]{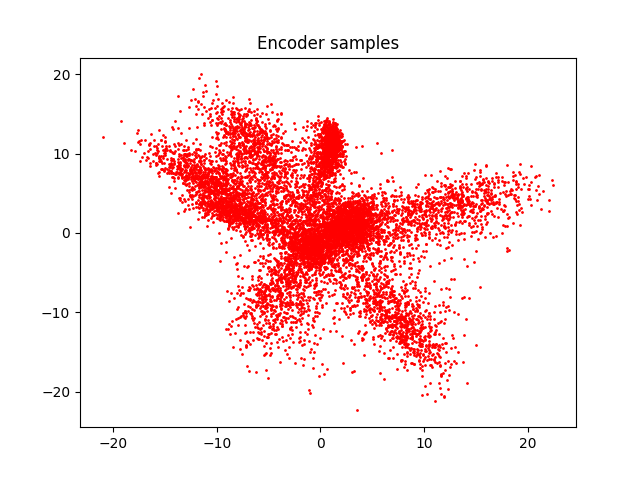} & \includegraphics[width=0.25\textwidth]{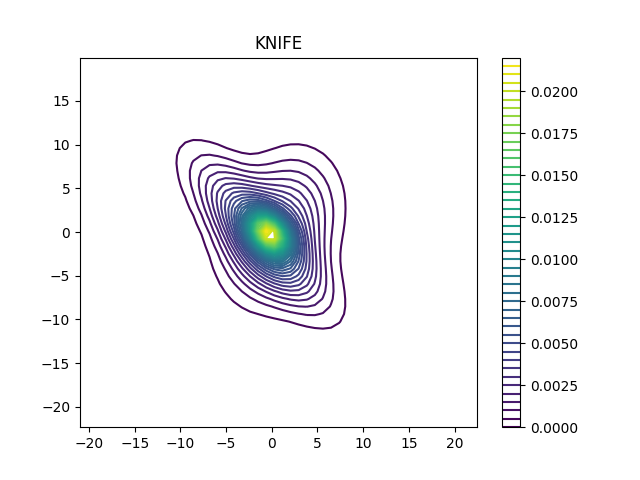} & \includegraphics[width=0.25\textwidth]{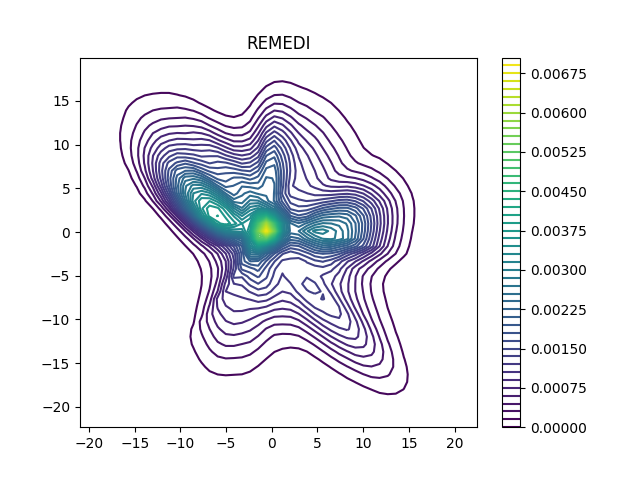} \\
    \hline
    20 & \includegraphics[width=0.25\textwidth]{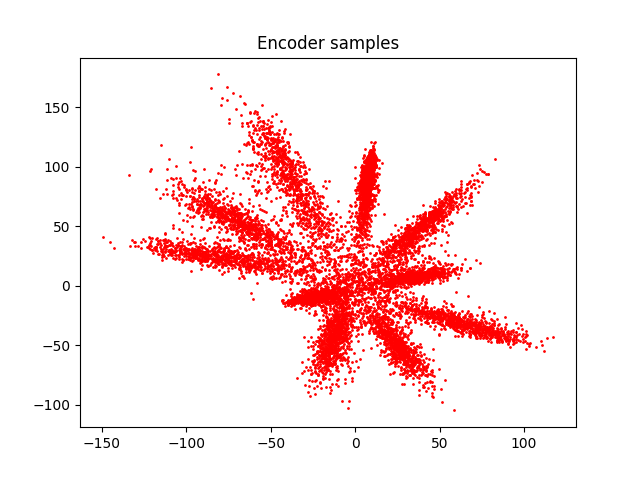} & \includegraphics[width=0.25\textwidth]{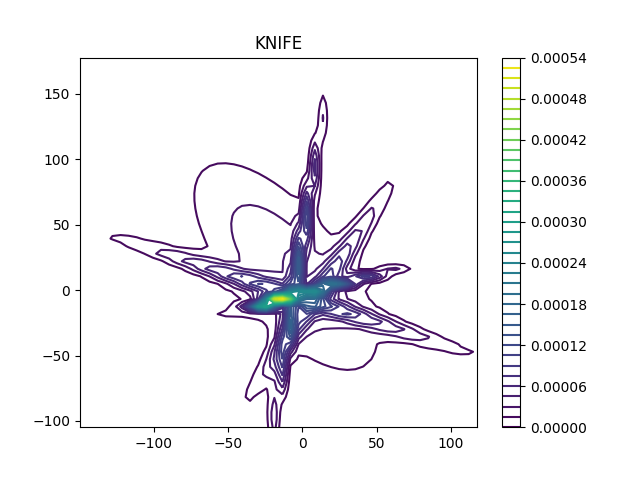} & \includegraphics[width=0.25\textwidth]{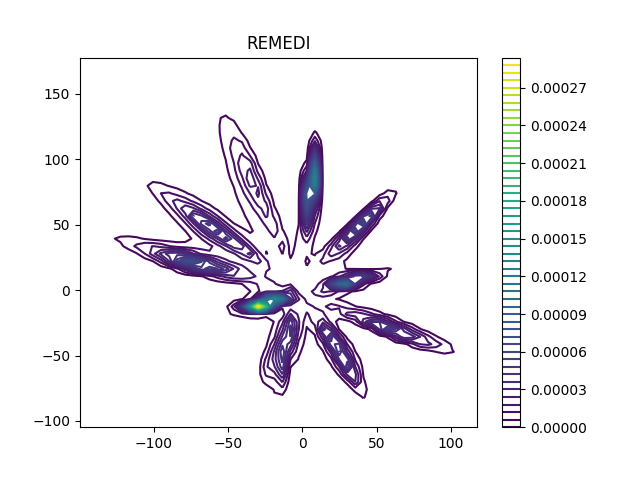} \\
    \hline
    40 & \includegraphics[width=0.25\textwidth]{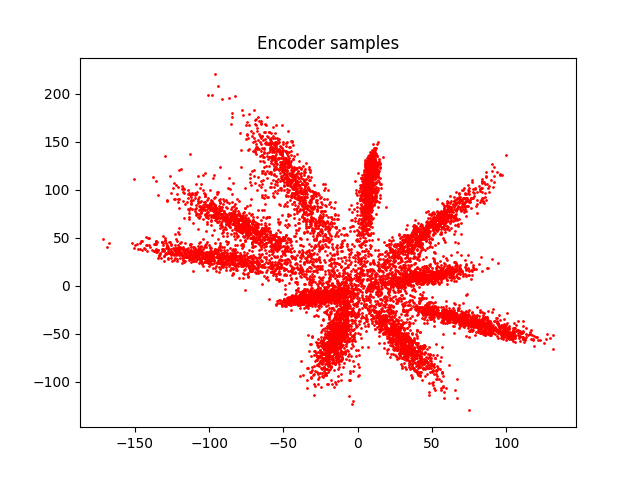} & \includegraphics[width=0.25\textwidth]{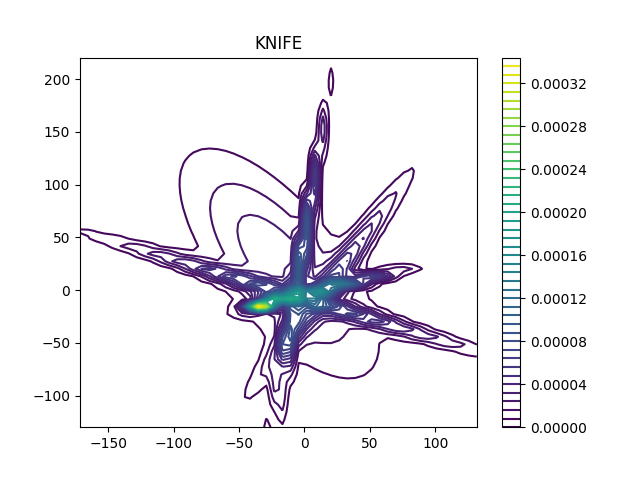} & \includegraphics[width=0.25\textwidth]{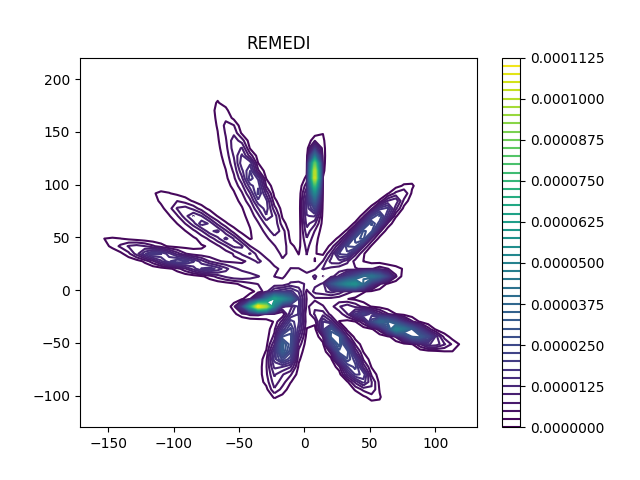} \\
    \hline
    60 & \includegraphics[width=0.25\textwidth]{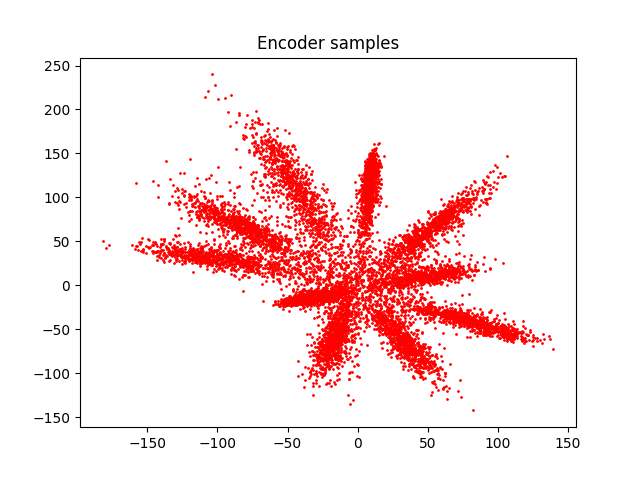} & \includegraphics[width=0.25\textwidth]{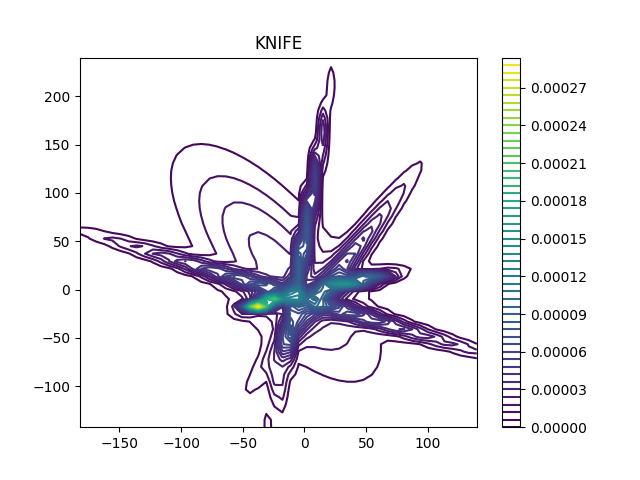} & \includegraphics[width=0.25\textwidth]{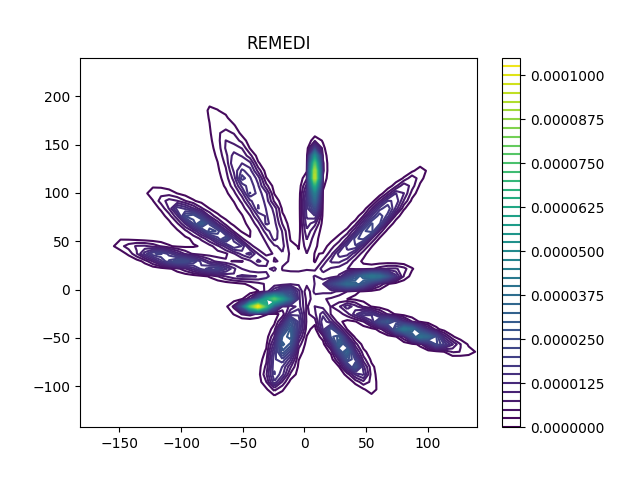} \\
    \hline
    90 & \includegraphics[width=0.25\textwidth]{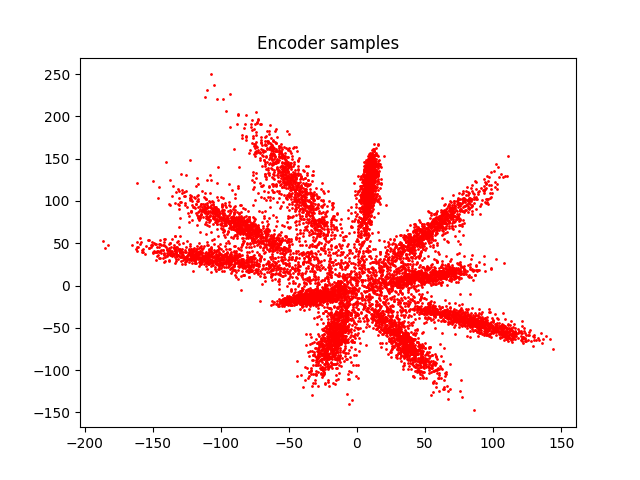} & \includegraphics[width=0.25\textwidth]{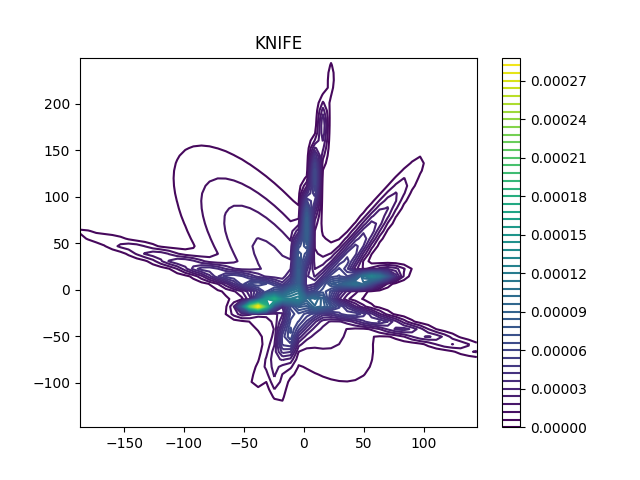} & \includegraphics[width=0.25\textwidth]{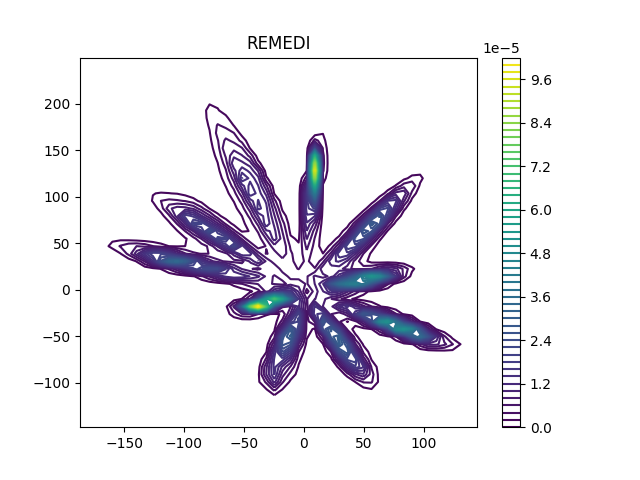} \\
    \hline
  \end{tabular}
  \caption{Sequence of latent space plots showing the evolution of the encoder samples and KNIFE and \texttt{REMEDI} density contours during training IB-REMEDI (KNIFE base) ($\beta=1e-09$) on MNIST.}
  \label{tab:latent_evol}
\end{table}

In Table~\ref{tab:latent_evol}, we plot different components of the 2-d latent space learned by IB-REMEDI with KNIFE base distribution throughout training on the MNIST dataset. The first column shows the latent space samples and the second and third columns show the KNIFE (10 components) and REMEDI contours as the training progresses. As the epoch increases, we observe that the latent space evolves into 10 clusters. We highlight that KNIFE struggles to learn the clusters especially when overlapping (e.g. epoch 0). To this end, REMEDI corrects the trained KNIFE and can locate the mass correctly around the clusters.

\subsection{Analysis of the latent space on CIFAR-10}
\label{sec:latentcifar10}
In this section, we present the 2-d latent space analysis on CIFAR-10. Similar to MNIST, we observe the samples from the latent space show clusters corresponding to the classes in Fig.~\ref{fig:encsamples_cifar}. In Fig.~\ref{fig:knifedensity_cifar}, we observe that the KNIFE fits the density well which is further improved by \texttt{REMEDI} corrections in Fig.~\ref{fig:REMEDI_cifar}. However, the gain from applying \texttt{REMEDI} from the KNIFE step is less in CIFAR-10 than it is on MNIST. 

\begin{figure}
    \centering
	\centering
	\begin{subfigure}{0.32\textwidth}
	\centering
	\includegraphics[width=\linewidth]{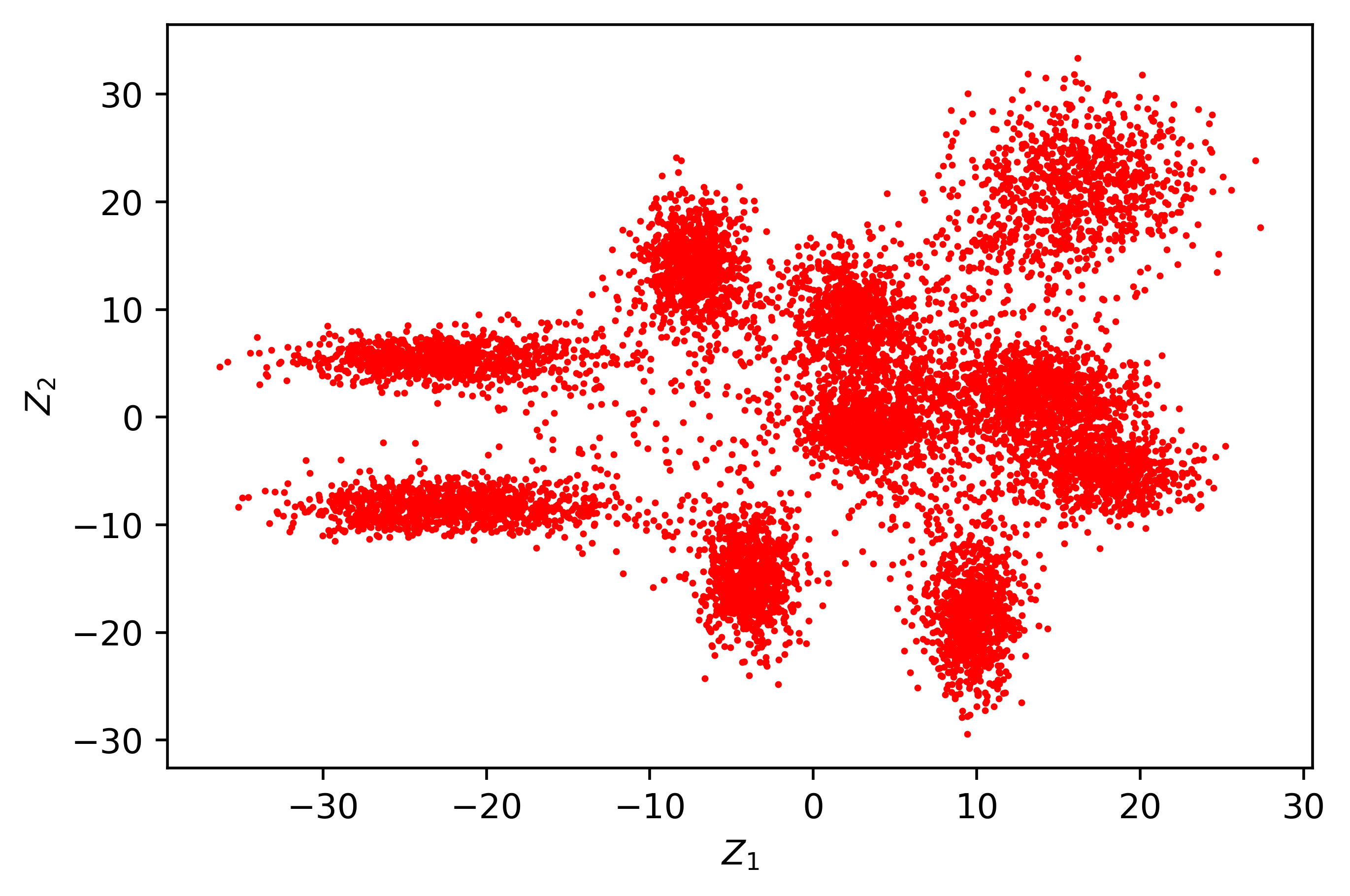}
	\caption{Encoder samples}
        \label{fig:encsamples_cifar}
	\end{subfigure}
	\begin{subfigure}{0.33\textwidth}
	\centering
	\includegraphics[width=\linewidth]{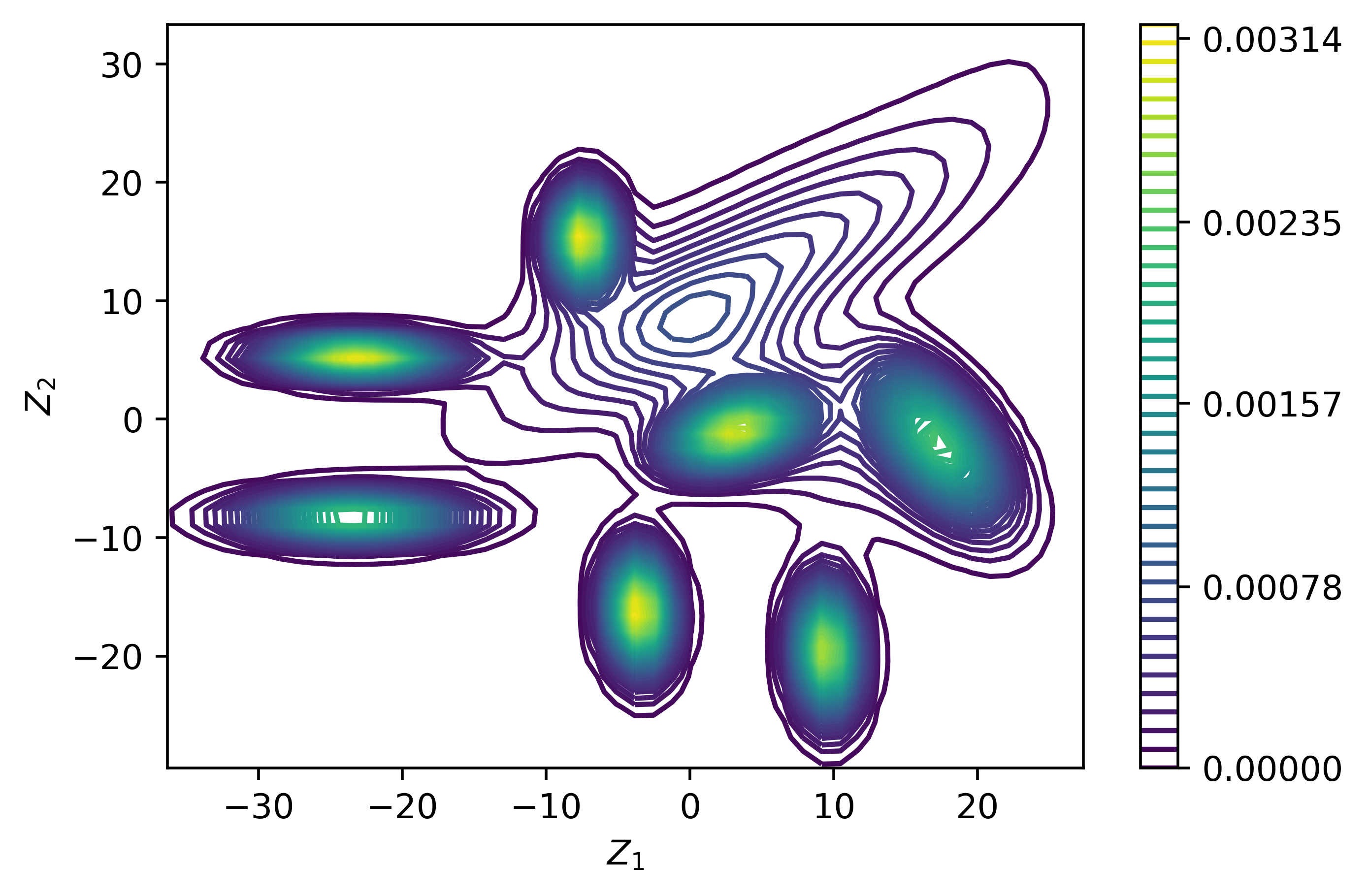}
	\caption{KNIFE contours}
        \label{fig:knifedensity_cifar}
	\end{subfigure}
	\begin{subfigure}{0.33\textwidth}
	\centering
	\includegraphics[width=\linewidth]{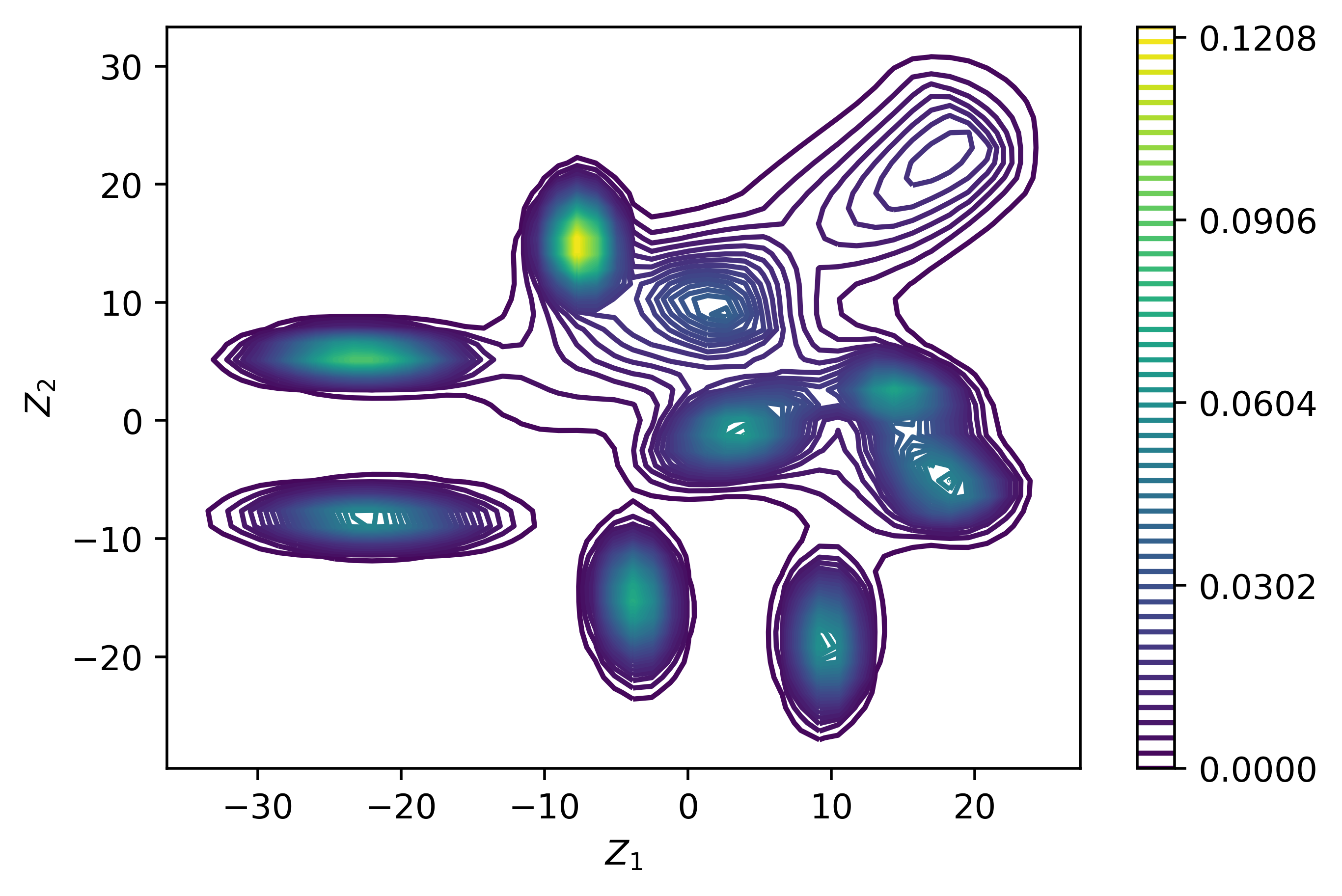}
	\caption{\texttt{REMEDI} contours}
        \label{fig:REMEDI_cifar}
	\end{subfigure}
    \caption{\texttt{REMEDI} marginal distribution of 2-d latent space on CIFAR-10.}
    \label{fig:latentspace_cifar}
\end{figure}

\end{document}